\documentclass{article}

 \usepackage[preprint]{neurips_2026}

\usepackage[utf8]{inputenc} % allow utf-8 input
\usepackage[T1]{fontenc}    % use 8-bit T1 fonts
\usepackage{url}            % simple URL typesetting
\usepackage{booktabs}       % professional-quality tables
\usepackage{amsfonts}       % blackboard math symbols
\usepackage{nicefrac}       % compact symbols for 1/2, etc.
\usepackage{microtype}      % microtypography
\usepackage{xcolor}         % colors
\usepackage[pdftex]{graphicx}
\usepackage{epstopdf}
\usepackage{multirow}
\usepackage{colortbl}
\usepackage[pagebackref=true,breaklinks=true,letterpaper=true,colorlinks,bookmarks=false]{hyperref}
\usepackage{subcaption}
\usepackage[namelimits]{amsmath} 
\usepackage{amssymb}            
\usepackage{mathrsfs}            
\usepackage{bm}
\usepackage{wrapfig}
\usepackage{bbm}
\usepackage{cleveref}
\usepackage{algorithm}
\usepackage{algorithmic}
\usepackage{marvosym}
\usepackage{inconsolata}

\newcommand{\eg}{\textit{e.g.}}

\makeatletter
\newcommand*\bigcdot{\mathpalette\bigcdot@{.5}}
\newcommand*\bigcdot@[2]{\mathbin{\vcenter{\hbox{\scalebox{#2}{$\m@th#1\bullet$}}}}}
\makeatother

\definecolor{c2}{HTML}{FBD9BD}
\definecolor{c3}{HTML}{fe793d}
\definecolor{c4}{HTML}{eedeb0}
\definecolor{pp}{HTML}{BC7FCD}
\definecolor{bb}{HTML}{CDE8E5}
\definecolor{rouse}{rgb}{0.981,0.961,0.941}

\crefname{section}{Sec.}{Secs.}
\Crefname{section}{Section}{Sections}
\Crefname{table}{Tab.}{Tabs.}
\crefname{table}{Table}{Tables}
\crefname{figure}{Fig.}{Figs.}
\crefname{equation}{Eq.}{Eqs.}

\usepackage{makecell}
\usepackage{tikz}
\usepackage{pgfplots}
\pgfplotsset{compat=1.17}
\usetikzlibrary{arrows.meta,positioning,calc,fit,backgrounds,shapes.geometric,decorations.pathreplacing}
\usepackage{colortbl}
\usepackage{enumitem}
\usepackage{wrapfig}
\usepackage{pifont}
\usepackage{mathtools}

\usepackage{amsthm}
\newtheorem{theorem}{Theorem}
\newtheorem{definition}{Definition}
\newtheorem{assumption}{Assumption}

\newcommand{\cmark}{\ding{51}}

\newcommand{\methodname}{RIDE}

\newcommand{\vs}{\textit{vs.}}
\newcommand{\Lretinex}{\mathcal{L}_{\text{ret}}}
\newcommand{\Lseg}{\mathcal{L}_{\text{seg}}}
\newcommand{\Lbound}{\mathcal{L}_{\text{bnd}}}
\newcommand{\Lcontrast}{\mathcal{L}_{\text{con}}}
\newcommand{\R}{\mathbb{R}}

\definecolor{bestcolor}{RGB}{255,100,100}
\definecolor{secondcolor}{RGB}{100,100,255}

\title{RIDE: Retinex-Informed Decoupling for \\ Exposing Concealed Objects}

\author{
 Chunming He$^{1}$\,,
 Rihan Zhang$^{1}$\,,
 Dingming Zhang$^{1}$\,,
 Chengyu Fang$^{2}$\,, \\
 \textbf{Longxiang Tang}$^{3}$\,,
 \textbf{Jingjia Feng}$^1$\,,
 \textbf{Fengyang Xiao}$^{1,*}$\,,
	\textbf{Sina Farsiu$^{1,*}$}\\
	$^1$Duke University, ~$^2$Tsinghua University, ~ $^3$Harvard University
\\
$*$ Corresponding Author,
Contact: chunming.he@duke.edu
 }

\begin{document}

\maketitle

\begin{abstract} \label{abstract}
Concealed Object Segmentation (COS) encompasses a family of dense-prediction tasks, including camouflaged object detection, polyp segmentation, transparent object detection, and industrial defect inspection, where targets are visually entangled with their surroundings through different physical mechanisms. Existing methods either operate directly on RGB images or employ \emph{heterogeneous} decompositions (\eg, Fourier, wavelet) that redistribute spatial evidence across scale/frequency coefficients, making pixel-aligned cues less direct. We introduce a fundamentally different perspective: \textbf{homogeneous image decomposition} via Retinex theory, which factorizes an image into illumination and reflectance components within the \emph{same} spatial domain. Our key insight is that visual entanglement enforces appearance matching in the composite space, but this does \emph{not} necessitate simultaneous matching in both component spaces, a phenomenon we formalize as the \textbf{Discriminability Gap Theorem}. Crucially, we show that across diverse COS sub-tasks, the underlying physical processes systematically anti-correlate illumination and reflectance differences, yielding theoretical guarantees that Retinex decomposition preserves or strictly improves total foreground--background discriminability across the full physical regime, with anti-correlation maximizing the gain. Building on this, we propose \textbf{RIDE} comprising: (i) a Task-Driven Retinex Decomposition module that learns segmentation-optimal factorizations end-to-end; (ii) a Discriminability Gap Attention mechanism that adaptively exploits where decomposition helps; and (iii) a Camouflage-Breaking Contrastive loss operating in reflectance feature space. Extensive experiments across four COS tasks and six broader segmentation settings demonstrate consistent improvements over established baselines. Comprehensive analyses validate our theoretical predictions and reveal interpretable patterns. Code will be released.
\end{abstract}

%% narrow the gap between equations and sentences
\setlength{\abovedisplayskip}{2pt}
\setlength{\belowdisplayskip}{2pt}
\section{Introduction} \label{introduction} %\vspace{-2mm}

Concealed Object Segmentation (COS) aims to identify objects that are visually blended into their surroundings through natural or artificial camouflage~\cite{fan2020camouflaged,le2019anabranch}. This makes COS fundamentally more challenging and has attracted growing attention due to its broad applications~\cite{le2019anabranch,fan2020pranet,he2025scaler,he2026refining}.

Unlike standard segmentation where targets exhibit clear visual contrast, camouflaged objects exploit texture matching, color assimilation, and structural mimicry, tissue homogeneity, transparency, or subtle material defects to achieve near-perfect visual homogeneity with their environments. Existing methods mainly focus on architectural engineering, such as multi-scale feature aggregation~\cite{fan2021concealed}, attention mechanisms~\cite{zhai2021mutual}, or edge-guided refinement~\cite{zhu2022can}, yet overlook a fundamental bottleneck: these approaches fundamentally search for discriminative signals within the same composite image space where camouflage is \textit{designed} to eliminate such signals.

A natural strategy to handle such visual entanglement is image decomposition, and frequency-domain methods (Fourier~\cite{zhong2022detecting}, wavelet~\cite{He2023Camouflaged}, DCT~\cite{cong2023frequency}) represent this direction. However, we argue that such \textit{heterogeneous} decompositions are suboptimal for COS. 
By mapping images from the spatial domain to scale/frequency coefficients, they redistribute spatial evidence in ways that make pixel-aligned cues less direct, weakening locality crucial for spatially-defined targets.
We instead advocate for \textit{homogeneous} decomposition that factorizes images within the same spatial domain, preserving spatial correspondence while exposing hidden discriminative signals (Fig.~\ref{fig:teaser}).

\begin{figure}[t]
\setlength{\abovecaptionskip}{0cm}
\centering
\includegraphics[width=\linewidth]{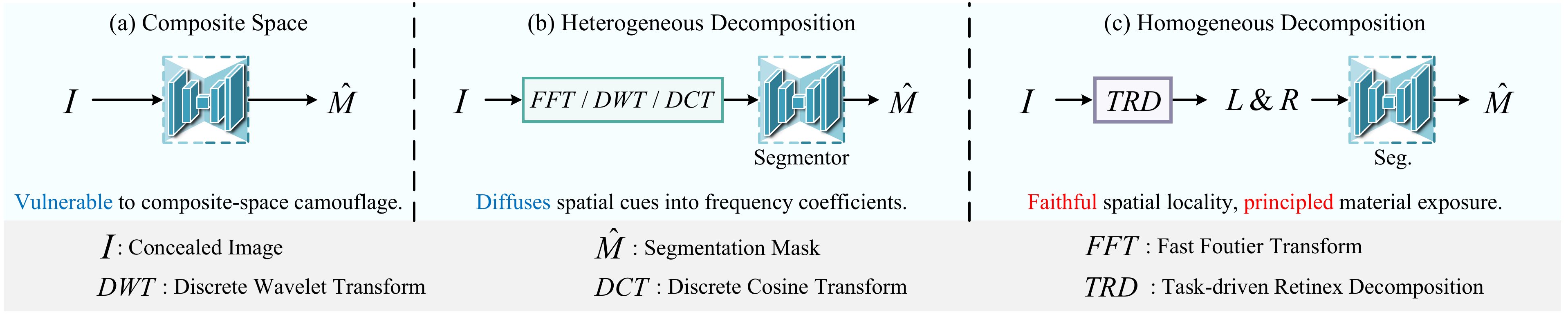}
\caption{{Heterogeneous \vs{} homogeneous decomposition for COS.} Heterogeneous methods (left) transform images to different representation spaces (frequency domain), scattering spatial information. Our homogeneous approach (right) decomposes within the same spatial domain via Retinex theory, preserving locality while exposing concealed material differences.}
\label{fig:teaser}
\vspace{-5mm}
\end{figure}

To this end, we propose \textbf{\methodname{}} (\textbf{R}etinex-\textbf{I}nformed \textbf{D}ecomposition for \textbf{E}xposing concealed objects), grounded in Retinex theory~\cite{land1971lightness}, which models image formation as $I = L \odot R$ with illumination $L$ and reflectance $R$. Our key insight is that across diverse COS sub-tasks, visual entanglement achieves composite appearance matching ($I_\text{fg} \approx I_\text{bg}$) without requiring simultaneous matching of both components. The physical processes governing each sub-task often produce \emph{anti-correlated} illumination--reflectance differences: (i) in COD, organisms select habitats with favorable lighting that compensates their reflectance signatures~\cite{stevens2009animal}; (ii) in polyp endoscopy, point-source illumination on curved mucosal surfaces creates strong illumination gradients that mask reflectance contrast between polyps and tissue; (iii) in transparent objects, specular reflections at glass boundaries produce illumination spikes where reflectance is neutral; (iv) in industrial defect inspection, directional lighting on textured surfaces creates illumination patterns that homogenize defects with the substrate. We formalize this as the \textbf{Discriminability Gap Theorem} (Theorem~\ref{thm:main}), proving that Retinex decomposition preserves or improves total foreground--background discriminability for any non-degenerate component configuration, with anti-correlation between $\delta_L$ and $\delta_R$ maximally amplifying the gain.

In the \textbf{Task-driven Retinex Decomposition (TRD)} module, \methodname{} learns segmentation-optimal factorizations end-to-end. A Mutual Exclusivity loss assigns edges to either illumination or reflectance, yielding clean material boundaries, while a Reflectance-Boundary head directly aligns reflectance gradients with ground-truth object boundaries. The \textbf{Discriminability Gap Attention (DGA)} module exploits the non-uniform gains from decomposition by estimating local discriminability gains and adaptively suppressing reliance on component features when no positive gap is detected. We further propose \textbf{Camouflage-Breaking Contrastive Learning (CBCL)} in reflectance space, where genuine material differences provide a physical foundation for foreground-background separation.

Our contributions are summarized as follows:

\begin{enumerate}[leftmargin=*,nosep]
    \item We unify a family of COS tasks (COD, polyp, glass, defect) under a single \textbf{homogeneous \vs{} heterogeneous decomposition} formulation, identifying their shared physical mechanism: anti-correlated illumination-reflectance differences.
    \item Grounded in Retinex theory, we formalize this mechanism via the \textbf{Discriminability Gap Theorem}, with assumptions explicitly stated and a complete proof in the appendix.
    \item We propose \textbf{\methodname{}}, integrating Task-Driven Retinex Decomposition, Discriminability Gap Attention, and reflectance-space contrastive learning into a unified, plug-in framework.
    \item Extensive experiments across four COS tasks, plus six broader segmentation tasks for generalization, validate our framework. In-depth analyses validate the theorem's predictions and reveal interpretable patterns of how decomposition exposes different forms of concealment.
\end{enumerate}

\section{Related Work}
\label{sec:related}
\noindent\textbf{Concealed Object Segmentation.} COS broadly covers tasks where targets share visual attributes with their surroundings, including camouflaged objects~\cite{fan2020camouflaged,le2019anabranch,xiao2024survey}, polyps~\cite{fan2020pranet}, transparent surfaces~\cite{mei2020don}, and industrial defects~\cite{he2025segment}. While each sub-task has historically been studied in isolation, they share the common challenge of visual entanglement. Recent COD advances focus on: (i) \textit{multi-scale reasoning}~\cite{pang2022zoom,jia2022segment}; (ii) \textit{foundation model adaptation}~\cite{chen2024sam,he2025segment}; (iii) \textit{frequency-domain analysis}~\cite{zhong2022detecting,He2023Camouflaged}; (iv) \textit{diffusion-based refinement}~\cite{he2025reversible,shen2025uncertainty}. Polyp segmentation is dominated by encoder-decoder designs~\cite{fan2020pranet} and transformer-based variants. Glass detection emphasizes boundary cues~\cite{mei2020don,he2021enhanced}. Industrial defect inspection often follows reconstruction-based or memory-bank paradigms. Our work is the first to unify these tasks under a homogeneous Retinex decomposition framework.

\noindent\textbf{Retinex Theory and Applications.}
Land and McCann~\cite{land1971lightness} proposed Retinex theory to model image formation as the product of illumination and reflectance. It has been widely applied to image enhancement and restoration~\cite{wei2018deep,wu2022uretinex,he2023reti,he2025unfoldir,he2025unfoldldm}, including learned~\cite{wei2018deep}, unfolding-based~\cite{wu2022uretinex,he2025unfoldir}, and diffusion-based~\cite{he2023reti} variants. Despite this success, Retinex has been largely unexplored for image segmentation. We are the first to systematically study Retinex decomposition for the broad family of COS tasks, providing both theoretical justification and a unified practical framework.

\noindent\textbf{Image Decomposition for Dense Prediction.}
Image decomposition has been explored for dense prediction. Intrinsic decomposition separates shading and albedo~\cite{barrow1978recovering,Xiao2026Beyond}; segmentation-oriented methods use texture--structure separation~\cite{xu2012structure,Xiao2026Quali}, edge-aware filtering~\cite{he2012guided,he2023HQG}, or learned disentanglement~\cite{dai2021disentangling}. Frequency-domain approaches further decompose features via Fourier~\cite{xu2020learning,rao2021global}, wavelet~\cite{He2023Camouflaged}, or learnable filters~\cite{chi2020fast}. However, these methods redistribute spatial evidence across frequency coefficients. Whereas COS targets are spatially defined, making homogeneous decomposition more suitable. The general principle that transforming observation spaces can increase foreground-background separability has long-standing roots in source separation (ICA~\cite{hyvarinen2000ica}) and discriminant analysis (trace-ratio LDA~\cite{wang2007trace,li2023sparse}). Our contribution is not this principle itself but its concrete instantiation for Retinex-based COS: a Retinex-specific condition (anti-correlated $\delta_L, \delta_R$) under which decomposition is provably beneficial, together with a task-driven architecture that learns such factorizations end-to-end.

\section{Theoretical Foundation: Discriminability Gap Analysis}
\label{sec:theory}

\subsection{Preliminaries and Notation}

Consider an image $I \in \R^{H \times W \times 3}$ containing a foreground object occupying region $\Omega_f$ and background region $\Omega_b$. Under the Retinex model, $I(p) = L(p) \odot R(p)$ for each pixel $p$, where $L(p) \in \R_{+}$ is a single-channel illumination and $R(p) \in [0,1]^3$ is RGB reflectance. We work in log-space: $\tilde{I}(p) = \tilde{L}(p) + \tilde{R}(p)$.

\begin{definition}[Regional Discriminability]
\label{def:disc}
For a (possibly vector-valued) component $X \in \{\tilde{I}, \tilde{L}, \tilde{R}\}$, the regional discriminability between foreground and background is
\begin{equation}
    D(X) = \|\mu_X^{f} - \mu_X^{b}\|_2^2 \,\Big/\, \big(\operatorname{tr}(\Sigma_X^{f}) + \operatorname{tr}(\Sigma_X^{b}) + \epsilon_R\big),
    \label{eq:discriminability}
\end{equation}
where $\mu_X^{r} = \mathbb{E}_{p \in \Omega_r}[X(p)]$ and $\Sigma_X^{r} = \operatorname{Cov}_{p \in \Omega_r}(X(p))$ for $r \in \{f, b\}$, and $\epsilon_R > 0$. This is the \emph{trace-ratio surrogate} of multivariate Fisher's discriminant criterion~\cite{wang2007trace,li2023sparse}: the numerator measures the squared foreground-background mean separation, while the denominator uses the total within-region scatter (trace of covariance) rather than the inverse covariance matrix in the standard Fisher form. This isotropic simplification avoids matrix inversion, remains numerically stable on dense feature maps, but does not model anisotropic covariance structure.
\end{definition}

\begin{definition}[Discriminability Gap]
The discriminability gap of component $X$ is $\Delta D(X) = D(X) - D(\tilde{I})$. A positive gap indicates that $X$ is more discriminative than the composite image.
\end{definition}

\subsection{The Discriminability Gap Theorem}
\begin{assumption}[Working assumptions]\label{ass:main}
(A1) The image follows the log-domain Retinex model $\tilde{I} = \tilde{L} + \tilde{R}$.
(A2) Within each region $\Omega_r$ ($r\in\{f,b\}$), $\tilde{L}$ and $\tilde{R}$ are conditionally independent, so $\operatorname{Cov}(\tilde{L},\tilde{R})_r \approx 0$.
(A3) The visual entanglement condition $D(\tilde{I}) \leq \epsilon$ holds for some small $\epsilon > 0$.
\end{assumption}

\begin{theorem}[Discriminability Gap Existence]\label{thm:main}
Under Assumption~\ref{ass:main}, let $\delta_L = \mu_{\tilde{L}}^f - \mu_{\tilde{L}}^b$ and $\delta_R = \mu_{\tilde{R}}^f - \mu_{\tilde{R}}^b$. Define the cosine angle between the two mean-difference vectors as
\[
    \rho \;=\; \cos\theta \;=\; \frac{\delta_L^\top \delta_R}{\|\delta_L\|\,\|\delta_R\|}\;\in\;[-1,1],
\]
and
\[
    \xi \;=\; \frac{\|\delta_L\|\,\|\delta_R\|}{\|\delta_L\|^2 + \|\delta_R\|^2}\;\in\;(0, \tfrac{1}{2}],
\]
where the upper bound on $\xi$ follows from AM--GM. Then, provided $\delta_L$ and $\delta_R$ are nonzero,
\begin{equation}
    D(\tilde{R}) + D(\tilde{L}) \;\geq\; D(\tilde{I}) \cdot \frac{1 + 2\xi}{1 + 2\rho\,\xi}.
    \label{eq:theorem}
\end{equation}
The factor $(1 + 2\xi)/(1 + 2\rho\xi)$ is no smaller than one for any $\rho \leq 1$, with equality only when $\rho = 1$ or one component has vanishing mean difference. Anti-correlation ($\rho < 0$) further amplifies the gain, explaining why COS sub-tasks, whose physical mechanisms systematically produce $\rho < 0$, benefit strongly from Retinex decomposition.
\end{theorem}

\begin{proof}[Proof Sketch]
From $\tilde{I} = \tilde{L} + \tilde{R}$, the mean differences satisfy $\delta_I = \delta_L + \delta_R$, so $\|\delta_I\|^2 = \|\delta_L\|^2 + \|\delta_R\|^2 + 2\|\delta_L\|\,\|\delta_R\|\,\rho$. Combined with the variance decomposition under (A2) and Titu's inequality (Cauchy--Schwarz in Engel form) applied to $\|\delta_L\|^2/W_L + \|\delta_R\|^2/W_R$, this yields the stated bound. The complete proof is in Appendix~\ref{app:proof}.
\end{proof}
 
\textbf{Physical interpretation across COS sub-tasks.} While Theorem~\ref{thm:main} guarantees decomposition gain for any $\rho \leq 1$ (with equality only in the trivial $\rho = 1$ case), the magnitude of the gain depends critically on $\rho$: anti-correlation maximizes it, while alignment trivializes it. COS sub-tasks systematically produce $\rho < 0$ through diverse physical mechanisms, which explains why decomposition is \emph{particularly} effective in this domain. We take COD and PIS tasks (two classic COS problems) as examples:
 
\begin{itemize}[leftmargin=*, nosep]
    \item \emph{Camouflaged Animals (COD):} Camouflage enforces $L_\text{fg} \odot R_\text{fg} \approx L_\text{bg} \odot R_\text{bg}$. Organisms select habitats with favorable lighting (\eg, dappled leaf litter), which produces anti-correlated $\delta_L,\delta_R$~\cite{stevens2009animal}.
    \item \emph{Polyp Endoscopy:} A point light source attached to the endoscope tip creates strong distance-dependent illumination on curved mucosal surfaces. Polyps and surrounding tissue have intrinsically distinct reflectance (vascularization, surface smoothness) that is masked by these illumination gradients, providing a textbook setting for $\rho < 0$.
\end{itemize}
 
This unifies COS tasks under one principle: visual entanglement defines each task and induces the anti-correlation exploited by Retinex decomposition. Building on this theoretical foundation, the next section instantiates the principle as a concrete network architecture and learning objective.

\section{Method: \methodname{} Framework}
\label{sec:method}

Based on the theoretical analysis in \S\ref{sec:theory}, we present the \methodname{} framework (Fig.~\ref{fig:framework}). \methodname{} comprises four key components: (1) Task-Driven Retinex Decomposition (\S\ref{sec:lrd}), (2) Triple-View Feature Encoder (\S\ref{sec:encoder}), (3) Discriminability Gap Attention (\S\ref{sec:dga}), and (4) Progressive Decoupled Decoder (\S\ref{sec:decoder}).

\begin{figure}[t]
\setlength{\abovecaptionskip}{0cm}
\centering
\includegraphics[width=\linewidth]{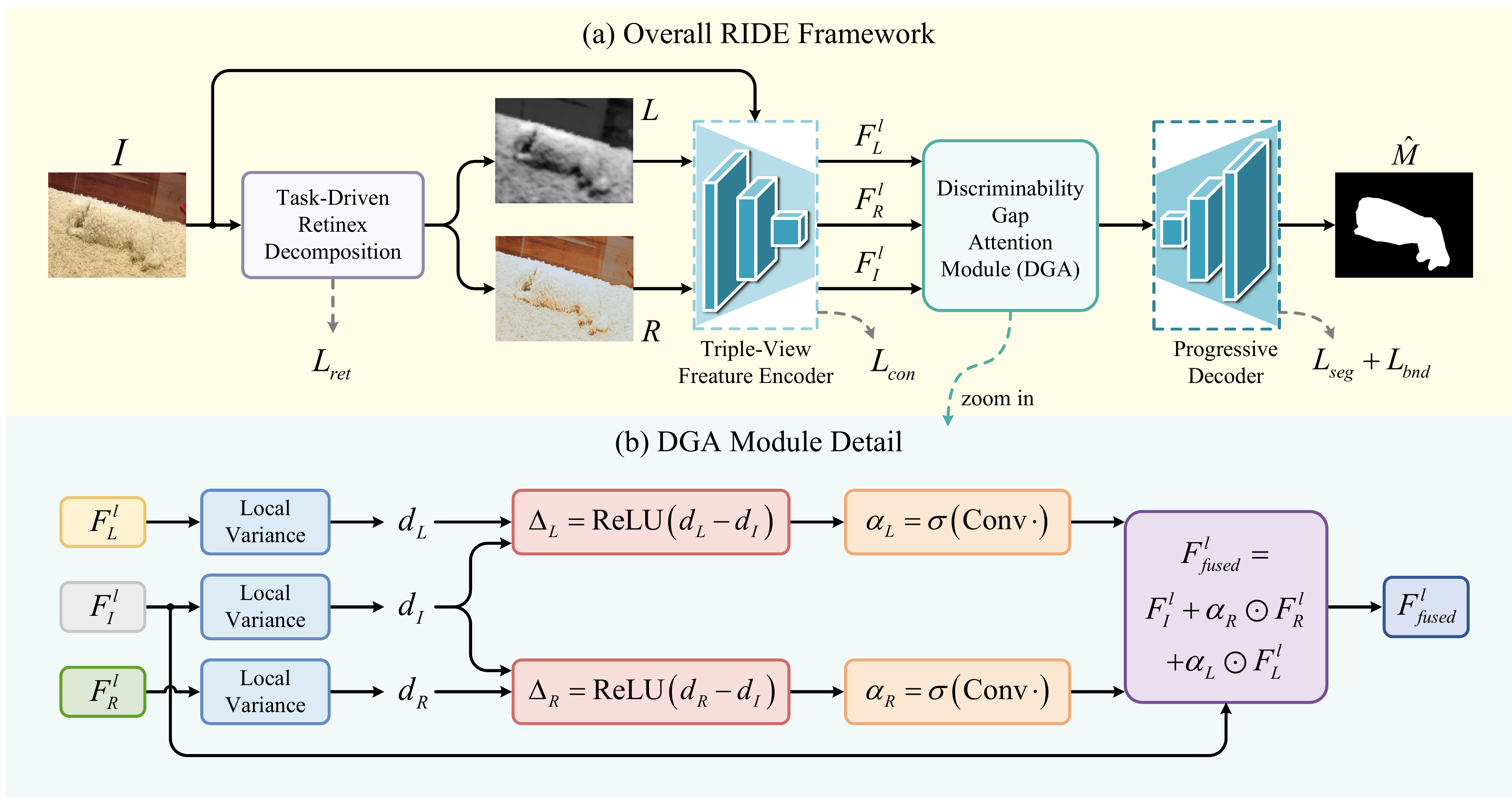}
\caption{{Overall architecture of \methodname{}.} 
}
\label{fig:framework}
\vspace{-1em}
\end{figure}

\subsection{Task-driven Retinex Decomposition (TRD)}
\label{sec:lrd}

Unlike conventional Retinex methods for image enhancement, our decomposition is jointly optimized for segmentation, learning to produce decompositions that maximally expose camouflaged objects.

\textbf{Architecture.} We employ a lightweight U-Net~\cite{ronneberger2015u} $\mathcal{G}_\theta$ with $\sim$1.2M parameters that takes image $I$ as input and produces illumination $L$ and reflectance $R$, formulated as:
\begin{equation}
    L, R = \mathcal{G}_\theta(I), \quad \text{s.t.} \quad L \in \R^{H\times W}_{+}, \;\; R \in [0,1]^{H\times W\times 3}
    \label{eq:decompose}
\end{equation}
We use \texttt{Sigmoid} for reflectance and \texttt{Softplus} for illumination to enforce valid ranges.

\textbf{Retinex-aware loss.} The decomposition is guided by physical priors through:
\begin{equation}
    \Lretinex = {\|I - L \odot R\|_1} + { \|\nabla L\|_2^2} + {\sum_c \|\nabla R_c\|_1} + {\mathcal{L}_\text{ME}},
    \label{eq:retinex_loss}
\end{equation}
where $c$ indexes color channels. $\nabla = (\nabla_h, \nabla_v)$ denotes spatial gradients along horizontal and vertical directions. $\mathcal{L}_\text{ME}$ represents Mutual Exclusivity (ME) loss, which encourages edges to be attributed to either illumination \textit{or} reflectance, but not both:
\begin{equation}
    \mathcal{L}_\text{ME} = \frac{1}{|\Omega|}\sum_{p \in \Omega} \sum_{d \in \{h,v\}} |\nabla_d L(p)| \cdot |\nabla_d R(p)|.
    \label{eq:me_loss}
\end{equation}
This is crucial for COS: clean reflectance edges indicate material (and potential object) boundaries, while shared edges cause ambiguity. The ME loss acts as a ``winner-take-all'' textural attribution.

\textbf{Task-driven training.} $\Lretinex$ is optimized jointly with the segmentation objective, allowing gradients from $\Lseg$ to propagate back to $\mathcal{G}_\theta$ and encourage \textit{segmentation-optimal} decompositions. Rather than seeking a perfect Retinex factorization, the goal is to maximize the discriminability gap.
\vspace{-1mm}
\subsection{Triple-View Feature Encoder}\vspace{-1mm}
\label{sec:encoder}
We extract multi-scale features from all three views ($I, L, R$) using a shared encoder $\Phi$ with view-specific lightweight adapters. For each view $V \in \{I, L, R\}$ and level $l \in \{1,2,3,4\}$, we get: 
\begin{equation}
    F_V^l = \Phi^l(V) + \mathcal{A}_V^l(\Phi^l(V)),
    \label{eq:adapter}
\end{equation}
where $\mathcal{A}_V^l$ is a common bottleneck adapter ($\text{Conv}_{1\times1}(C_l, C_l/r) \to \text{GELU} \to \text{Conv}_{1\times1}(C_l/r, C_l)$, $r=4$). The shared backbone ensures efficiency, while adapters capture view-specific characteristics: $\mathcal{A}_I$ for composite appearance, $\mathcal{A}_L$ for illumination patterns, and $\mathcal{A}_R$ for material-related features. The backbone parameters $\Phi$ are tied across views, and implementation-level optimizations (shared activations across views where applicable, fused adapter convolutions) keep the per-view marginal cost small relative to single-view baselines, as reflected in our wall-clock measurements (Table~\ref{table:efficiency}).
\vspace{-1mm}
\subsection{Discriminability Gap Attention}\vspace{-1mm}
\label{sec:dga}

While Definition~\ref{def:disc} measures regional discriminability through inverse within-region scatter (lower variance is better for compact, separable regions), DGA operates at a different scale: regional separability manifests \emph{locally} as boundary transitions where feature contrast spikes. We use \emph{local feature contrast} as a differentiable proxy for these transitions. Before computing local contrast, $F_V^l$ is feature-normalized to bound channel-wise scale, preventing the gap estimator from being driven by branch-specific feature magnitudes. For each spatial position $p$ and level $l$:
\begin{equation}
    d_V^l(p) = \frac{1}{|\Omega_k|} \sum_{q \in \Omega_k(p)} \|F_V^l(q) - \bar{F}_V^l(p)\|_2^2, \quad V \in \{I, L, R\}
    \label{eq:local_var}
\end{equation}
where $\bar{F}_V^l(p)$ is the local mean. High variance indicates potential boundary transitions. We then compute the discriminability gap, measuring each component's improvement over the composite:
\begin{equation}
    \Delta_R^l(p) = \text{ReLU}(d_R^l(p) - d_I^l(p)), \quad \Delta_L^l(p) = \text{ReLU}(d_L^l(p) - d_I^l(p)).
    \label{eq:gap}
\end{equation}
Positive $\Delta_R^l(p)$ indicates reflectance transitions invisible in the composite image. These gaps are converted into spatial attention weights $\alpha_R^l = \text{Sigmoid}(\text{Conv}_{3\times3}(\text{Conv}_{1\times1}(\text{Concate}(\Delta_R^l, d_R^l, d_I^l))))$ and $\alpha_L^l$ analogously, which guide the final fusion:
\begin{equation}
    F_\text{fused}^l = F_I^l + \alpha_R^l \odot F_R^l + \alpha_L^l \odot F_L^l.
    \label{eq:fusion}
\end{equation}
\textbf{When no positive discriminability gap is detected} (\eg, in the trivial $\rho \to 1$ regime predicted by Theorem~\ref{thm:main}, or when both $\|\delta_L\|, \|\delta_R\| \to 0$ as in absolute material mimicry under matched lighting), the gap-conditioned attention weights $\alpha_R^l, \alpha_L^l$ are driven toward small values, suppressing reliance on the component features so that DGA recovers composite-only behavior $F_\text{fused}^l \approx F_I^l$. {Otherwise}, component features with positive discriminability gains are selectively enhanced.

\vspace{-1mm}
\subsection{Progressive Decoupled Decoder}\vspace{-1mm}
\label{sec:decoder}

The decoder progressively integrates DGA-fused features: $D^l = \text{ConvBlock}(\text{Up}(D^{l+1}) + F_\text{fused}^l)$ for $l = 3, 2, 1$, where $D^4 = F_\text{fused}^4$ and $\text{ConvBlock}$ consists of two $3\times3$ convolutions with batch normalization and GELU. Three prediction heads operate on $D^1$: a {segmentation head} $\hat{M} = \sigma(\text{Conv}_{1\times1}(D^1))$ as the primary output, a {boundary head} $\hat{B} = \sigma(\text{Conv}_{1\times1}(D^1))$ for edge sharpening, and a {reflectance-boundary head} $\hat{B}_R = \sigma(\text{Conv}_{1\times1}(\text{Conv}_{3\times3}(|\nabla R|)))$ that predicts boundaries directly from reflectance gradients, linking decomposition quality to segmentation accuracy. 

\vspace{-1mm}
\subsection{Loss Functions}\vspace{-1mm}
\label{sec:loss}

\textbf{Segmentation loss.} We use a combination of binary cross-entropy and IoU loss with deep supervision:
\begin{equation}
    \Lseg = \sum_{l=1}^{4} \frac{1}{2^{l-1}}\left[\mathcal{L}_\text{BCE}(\hat{M}^l, M^l_\text{GT}) + \mathcal{L}_\text{IoU}(\hat{M}^l, M^l_\text{GT})\right].
    \label{eq:seg_loss}
\end{equation}

\textbf{Boundary loss.} We supervise both boundary heads using the ground truth boundary $B_\text{GT}$:
\begin{equation}
    \Lbound = \mathcal{L}_\text{BCE}(\hat{B}, B_\text{GT}) +
    % \beta_R \cdot 
    \mathcal{L}_\text{BCE}(\hat{B}_R, B_\text{GT}).
    \label{eq:bnd_loss}
\end{equation}
The reflectance-boundary loss $\mathcal{L}_\text{BCE}(\hat{B}_R, B_\text{GT})$ directly links reflectance gradients with object boundaries, providing explicit supervision for task-driven decomposition.

\textbf{Camouflage-breaking contrastive loss.} We apply a corrected InfoNCE in reflectance feature space using two augmented views of the same input, providing a non-degenerate positive pair:
\begin{equation}
    \Lcontrast = -\log \frac{\exp(\text{sim}(f_{R,a}^+, f_{R,b}^+) / \tau)}{\exp(\text{sim}(f_{R,a}^+, f_{R,b}^+) / \tau) + \sum_{j=1}^{J} \exp(\text{sim}(f_{R,a}^+, f_R^{-,j}) / \tau)},
    \label{eq:contrast_loss}
\end{equation}
where $\tau$ is temperature, $f_{R,a}^+$ and $f_{R,b}^+$ are foreground reflectance features from two augmented views of the same image, computed via normalized masked global average pooling
\[
    f_{R,a}^+ = \frac{\sum_p F_{R,a}^l(p)\,M_\text{GT}^l(p)}{\sum_p M_\text{GT}^l(p) + \epsilon},
\]
$f_R^{-,j}$ are negatives from $J$ randomly sampled background patches and $\text{sim}(\cdot,\cdot)$ is cosine similarity.

\textbf{Why $R$-space rather than $I$-space?} In $I$-space, visual entanglement enforces $f_I^+ \approx f_I^-$, so the InfoNCE positive-negative gap collapses and the loss provides little discriminative signal. The reflectance space, by contrast, encodes intrinsic material differences that survive entanglement (Theorem~\ref{thm:main}), giving contrastive learning a physically grounded foundation for foreground-background separation. We empirically verify this in~\cref{sec:analysis}.

\textbf{Total loss.} The complete objective is formulated as:
\begin{equation}
    \mathcal{L} = \Lseg +  \Lretinex +  \Lbound + \Lcontrast.
    \label{eq:total_loss}
\end{equation}

\vspace{-1mm}
\section{Experiments}\vspace{-1mm}
\label{sec:exp}
\textbf{Implementation details.} Our RIDE is implemented in PyTorch on two RTX4090 GPUs. Following \cite{He2023Camouflaged,he2025run}, our encoder  (ResNet50 by default) is pre-trained on ImageNet~\cite{deng2009imagenet}. During the training phase, we select Adam with momentum terms $(0.9,0.999)$ and resize all image as $352\times 352$. We set the batch size as 36 and initialize the learning rate as 0.0001, dividing by 10 every 80 epochs.
The TRD module uses a 5-level U-Net with channels $\{32, 64, 128, 64, 32\}$. For DGA, we set $k=7$ for local variance computation. Contrastive loss temperature $\tau = 0.1$.
All results are from published papers or reproduced using official code. Datasets and metrics are shown in Appendix \ref{app:datasets}.

\vspace{-1mm}
\subsection{Comparative evaluation}\vspace{-1mm}
\label{sec:sota}

\begin{table*}[tbp!]
\begin{minipage}{1\textwidth}
\setlength{\abovecaptionskip}{0cm} 
		\setlength{\belowcaptionskip}{0cm}
		\centering
            \caption{Results on camouflaged object detection.  
            } \label{table:CODQuanti}
            % \vspace{-2mm}
		\resizebox{\columnwidth}{!}{
			\setlength{\tabcolsep}{1.4mm}
			\begin{tabular}{l|c|cccc|cccc|cccc|cccc} 
				\toprule
\multicolumn{1}{c|}{}                                        & \multicolumn{1}{c|}{}                           & \multicolumn{4}{c|}{\textit{CHAMELEON} }                                                                                                                                         & \multicolumn{4}{c|}{\textit{CAMO} }                                                                                                                                             & \multicolumn{4}{c|}{\textit{COD10K} }                                                                                                                                          & \multicolumn{4}{c}{\textit{NC4K} }                                                                                                                        \\ \cline{3-18} 
\multicolumn{1}{l|}{\multirow{-2}{*}{Methods}} & \multicolumn{1}{c|}{\multirow{-2}{*}{Backbones}} & {\cellcolor{gray!40}$M$~$\downarrow$}                                  & {\cellcolor{gray!40}$F_\beta$~$\uparrow$}                               & {\cellcolor{gray!40}$E_\phi$~$\uparrow$}                               & \multicolumn{1}{c|}{\cellcolor{gray!40}$S_\alpha$~$\uparrow$}                                   & {\cellcolor{gray!40}$M$~$\downarrow$}                                  & {\cellcolor{gray!40}$F_\beta$~$\uparrow$}                               & {\cellcolor{gray!40}$E_\phi$~$\uparrow$}                               & \multicolumn{1}{c|}{\cellcolor{gray!40}$S_\alpha$~$\uparrow$}                                   & {\cellcolor{gray!40}$M$~$\downarrow$}                                  & {\cellcolor{gray!40}$F_\beta$~$\uparrow$}                               & {\cellcolor{gray!40}$E_\phi$~$\uparrow$}                               & \multicolumn{1}{c|}{\cellcolor{gray!40}$S_\alpha$~$\uparrow$}                                   & {\cellcolor{gray!40}$M$~$\downarrow$}                                  & {\cellcolor{gray!40}$F_\beta$~$\uparrow$}                               & {\cellcolor{gray!40}$E_\phi$~$\uparrow$}                               & \multicolumn{1}{c}{\cellcolor{gray!40}$S_\alpha$~$\uparrow$}                                   \\ \midrule 
\multicolumn{1}{l|}{FEDER~\cite{He2023Camouflaged}} & \multicolumn{1}{c|}{ResNet50}  & { {0.028}} & { {0.850}} & 0.944 & \multicolumn{1}{c|}{0.892} & {{0.070}} & {0.775} & 0.870 & \multicolumn{1}{c|}{0.802} & 0.032 & 0.715 & 0.892 & \multicolumn{1}{c|}{0.810} & {{0.046}} & {{0.808}} & {{0.900}} & {{0.842}} \\
\multicolumn{1}{l|}{FGANet~\cite{zhaiexploring}}                     & \multicolumn{1}{c|}{ResNet50}                   & 0.030                                 & 0.838                                 & {{0.945}}                                 & 0.891                                 & 0.070  & 0.769  & 0.865  & \multicolumn{1}{c|}{0.800}  & 0.032   & 0.708 & 0.894  & \multicolumn{1}{c|}{0.803}                                 & 0.047                                & 0.800                                 & 0.891                                 & 0.837                                 \\
\multicolumn{1}{l|}{FocusDiff~\cite{zhao2025focusdiffuser}} & \multicolumn{1}{c|}{ResNet50} & {{0.028}} & 0.843 & 0.938 & 0.890  & {{0.069}} & 0.772 & {{0.883}} & {{0.812}} & {{0.031}} & {{0.730}} & {{0.897}} & 0.820 & {{0.044}} & {{0.810}} & {{0.902}} &{{0.850}}    \\ 
\multicolumn{1}{l|}{FSEL~\cite{sun2025frequency}} & \multicolumn{1}{c|}{ResNet50} & 0.029 & 0.847 & 0.941 & {{0.893}}  & {{{0.069}}} & {{0.779}} & {0.881} & {{{0.816}}} & 0.032 & 0.722 & 0.891 & {{0.822}} & 0.045 & 0.807 & 0.901 & 0.847    \\ 
\multicolumn{1}{l|}{RUN~\cite{he2025run}}  & \multicolumn{1}{c|}{ResNet50} & {{0.027}} & {{0.855}} & {{0.952}} & {{0.895}} & 0.070 & {{0.781}} & 0.868 & 0.806 & {0.030} & {{0.747}} & {{0.903}} & {{0.827}} & {0.042} & {{0.824}} & {{0.908}} & {{0.851}}     \\
\rowcolor{c2!20} RIDE (Ours) & ResNet50 & \textbf{0.026} & \textbf{0.868} & \textbf{0.958} & \textbf{0.900} & \textbf{0.068} & \textbf{0.790} & \textbf{0.891} & \textbf{0.817} & \textbf{0.028} & \textbf{0.763} & \textbf{0.912} & \textbf{0.834} & \textbf{0.041} & \textbf{0.831} & \textbf{0.916} & \textbf{0.858} \\
                \midrule
\multicolumn{1}{l|}{BSA-Net~\cite{zhu2022can}}            & \multicolumn{1}{c|}{Res2Net50}                  & 0.027                                 & 0.851                                 & 0.946                                 & \multicolumn{1}{c|}{0.895}                                 & 0.079                                 & 0.768                                 & 0.851                                 & \multicolumn{1}{c|}{0.796}                                 & 0.034                                 & 0.723                                 & 0.891                                 & \multicolumn{1}{c|}{0.818}                                 & 0.048                                 & 0.805                                 & 0.897                                 & 0.841                                 \\
\multicolumn{1}{l|}{RUN~\cite{he2025run}}  & \multicolumn{1}{c|}{Res2Net50} & {0.024} & {0.879} & {0.956} & {0.907} & {{0.066}} & {0.815} & {0.905} & {{0.843}}  & {0.028} & {0.764} & {{0.914}} & {{0.849}} & {0.041} & {0.830} & {0.917} &0.859                \\
\rowcolor{c2!20} RIDE (Ours) & Res2Net50 & \textbf{0.023} & \textbf{0.886} & \textbf{0.960} & \textbf{0.911} & \textbf{0.065} & \textbf{0.826} & \textbf{0.912} & \textbf{0.848} & \textbf{0.025} & \textbf{0.783} & \textbf{0.922} & \textbf{0.855} & \textbf{0.039} & \textbf{0.838} & \textbf{0.924} & \textbf{0.866}
                \\ 
                \midrule
CamoDiff~\cite{sun2025conditional} & PVT V2 & 0.022 & 0.868 & 0.952 & 0.908 & {0.042} & 0.853 & 0.936 & 0.878 & {0.019} & 0.815 & 0.943 & 0.883 & {0.028} & 0.858 & 0.942 & 0.895 \\
\multicolumn{1}{l|}{RUN~\cite{he2025run}} &  \multicolumn{1}{c|}{PVT V2} & {{0.021}} & {0.877} & {{0.958}} & {{0.916}} & 0.045 &  {0.861} &  {0.934} & {0.877} & {0.021} &  0.810 & {0.941} &  {0.878} & {0.030} & {0.868} & {0.940} & {0.892} \\ 
\rowcolor{c2!20} RIDE (Ours) & PVT V2 & \textbf{0.020} & \textbf{0.886} & \textbf{0.965} & \textbf{0.920} & \textbf{0.041} & \textbf{0.868} & \textbf{0.940} & \textbf{0.887} & \textbf{0.018} & \textbf{0.838} & \textbf{0.950} & \textbf{0.889} & \textbf{0.027} & \textbf{0.883} & \textbf{0.948} & \textbf{0.898}
                \\ \midrule
VSCode~\cite{luo2024vscode} & DINOv2-L & 0.020 & 0.882 & 0.954 & 0.915 & 0.043 & 0.860 & 0.931 & 0.882 & 0.020 & 0.818 & 0.935 & 0.880 & 0.030 & 0.867 & 0.940 & 0.895 \\
C3Net~\cite{jan2025c3net} & DINOv2-L & 0.016 & 0.905 & 0.968 & 0.927 & 0.031 & 0.892 & 0.953 & 0.904 & 0.016 & 0.852 & 0.961 & 0.898 & 0.024 & 0.894 & 0.958 & 0.913 \\
\rowcolor{c2!20} {{RIDE (Ours)}} & {\textit{DINOv2-L}} & \textbf{0.014} & \textbf{0.920} & \textbf{0.978} & \textbf{0.933} & \textbf{0.028} & \textbf{0.911} & \textbf{0.967} & \textbf{0.912} & \textbf{0.013} & \textbf{0.887} & \textbf{0.977} & \textbf{0.906} & \textbf{0.019} & \textbf{0.914} & \textbf{0.969} & \textbf{0.921} \\
                \bottomrule            
		\end{tabular}}  
\end{minipage}
\vspace{-4mm}
\end{table*} 
\begin{figure*}[t]
\newlength{\subw}\setlength{\subw}{0.116\textwidth}  %<<< Change this single number to resize ALL sub-images in THIS figure
\begin{minipage}{\textwidth}
\setlength{\abovecaptionskip}{0cm}
\setlength{\belowcaptionskip}{0cm}
\centering
\begin{subfigure}{\subw}
\centering
\includegraphics[width=\textwidth]{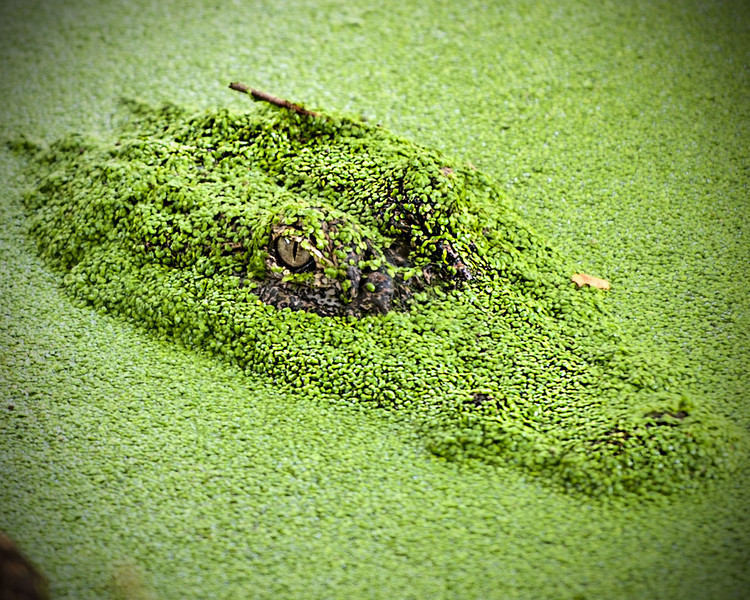}\vspace{-2pt}
\end{subfigure}
\begin{subfigure}{\subw}
\centering
\includegraphics[width=\textwidth]{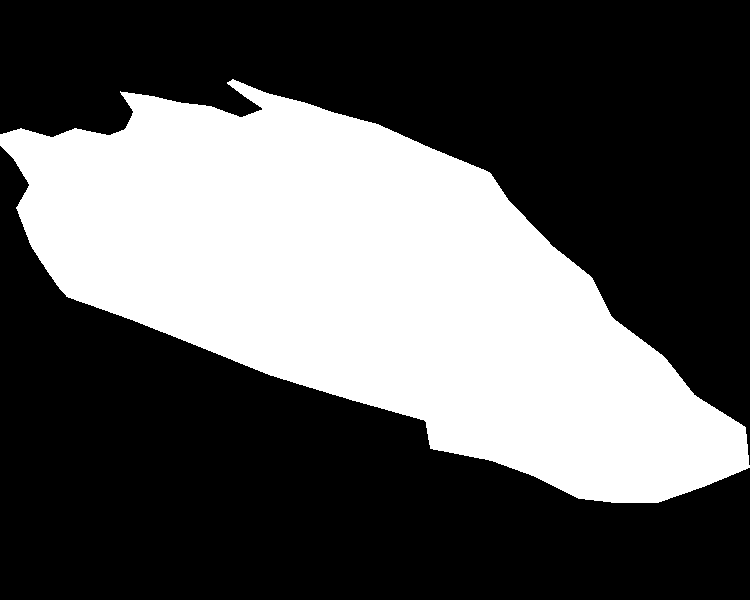}\vspace{-2pt}
\end{subfigure}
\begin{subfigure}{\subw}
\centering
\includegraphics[width=\textwidth]{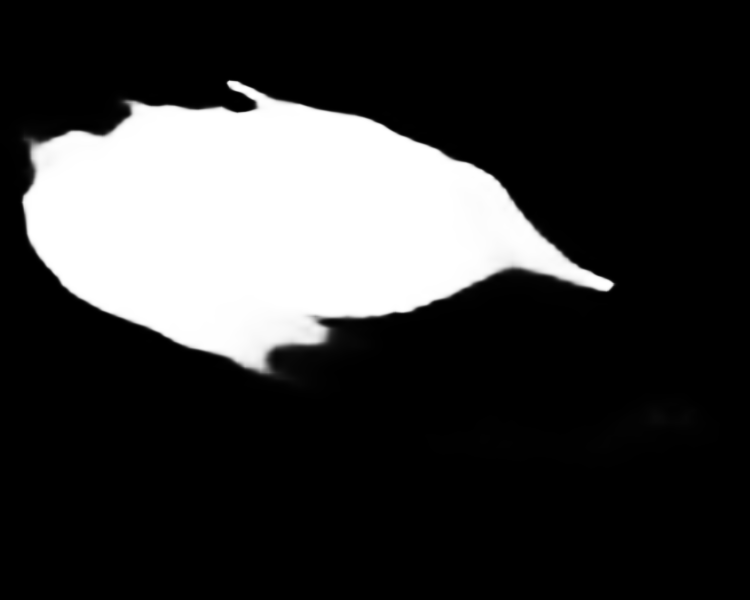}\vspace{-2pt}
\end{subfigure}
\begin{subfigure}{\subw}
\centering
\includegraphics[width=\textwidth]{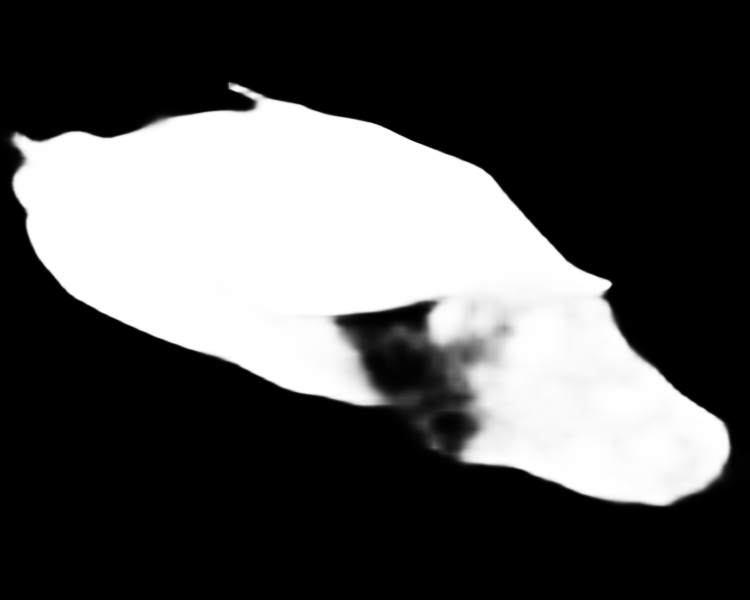}\vspace{-2pt}
\end{subfigure}
\begin{subfigure}{\subw}
\centering
\includegraphics[width=\textwidth]{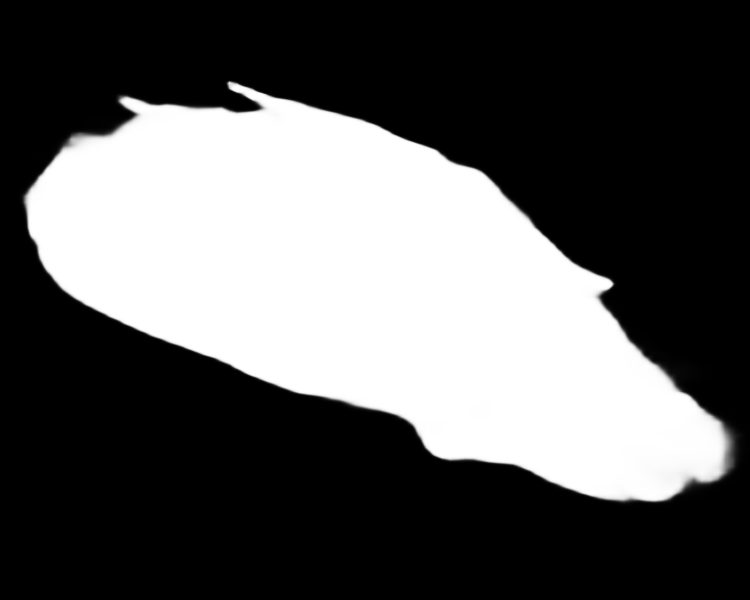}\vspace{-2pt}
\end{subfigure}
\begin{subfigure}{\subw}
\centering
\includegraphics[width=\textwidth]{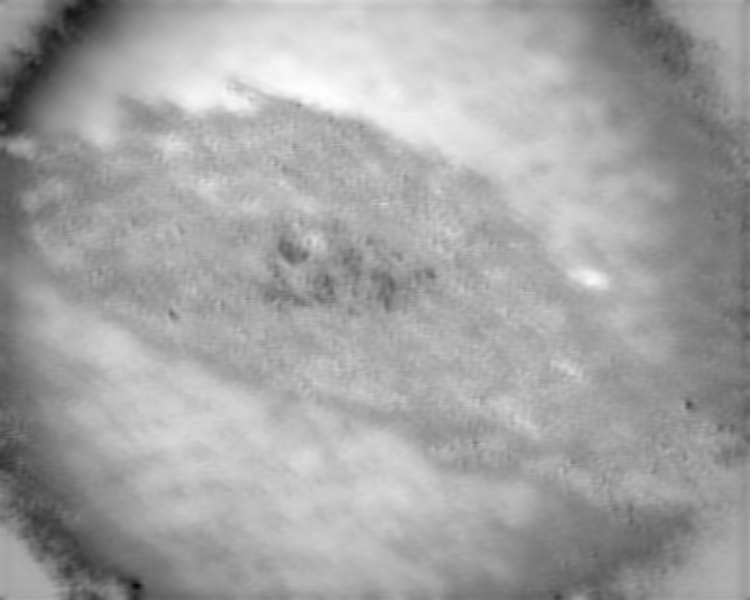}\vspace{-2pt}
\end{subfigure}
\begin{subfigure}{\subw}
\centering
\includegraphics[width=\textwidth]{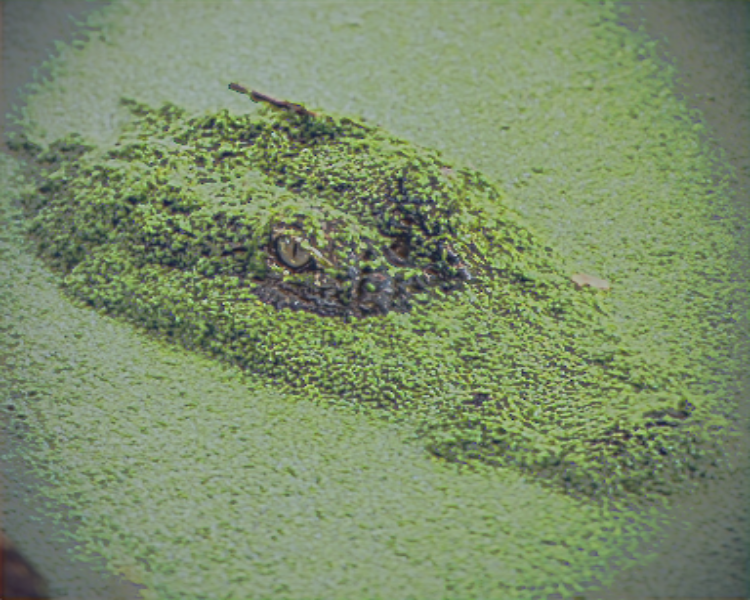}\vspace{-2pt}
\end{subfigure}
\begin{subfigure}{\subw}
\centering
\includegraphics[width=\textwidth]{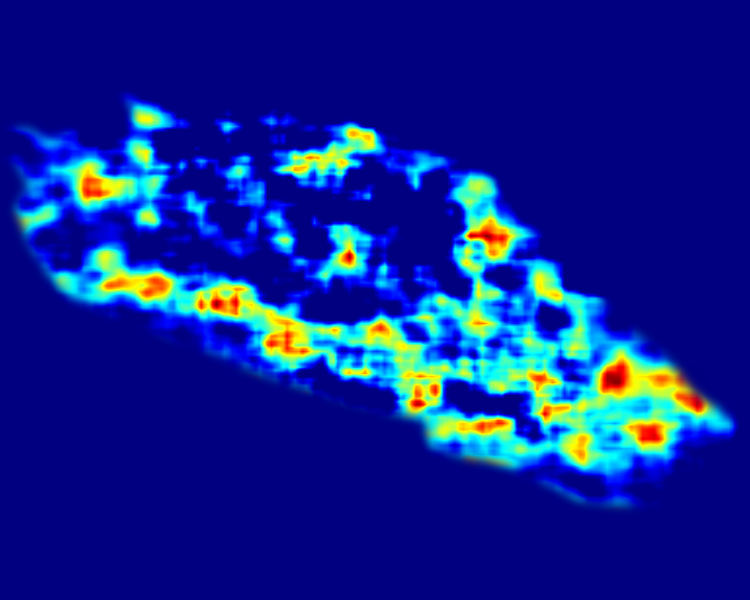}\vspace{-2pt}
\end{subfigure}\\ \vspace{1mm}
\begin{subfigure}{\subw}
\centering
\includegraphics[width=\textwidth]{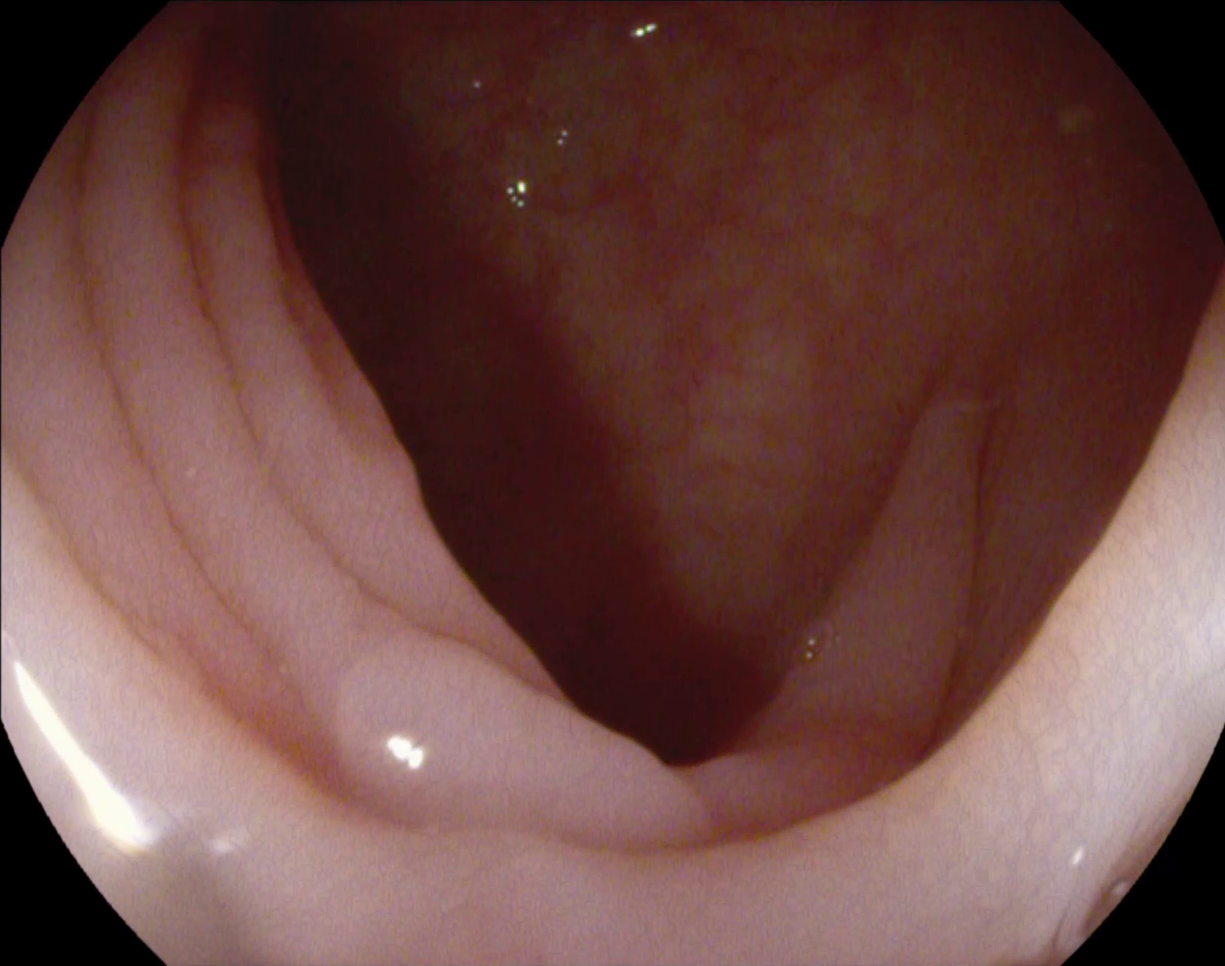}\vspace{-2pt}
\end{subfigure}
\begin{subfigure}{\subw}
\centering
\includegraphics[width=\textwidth]{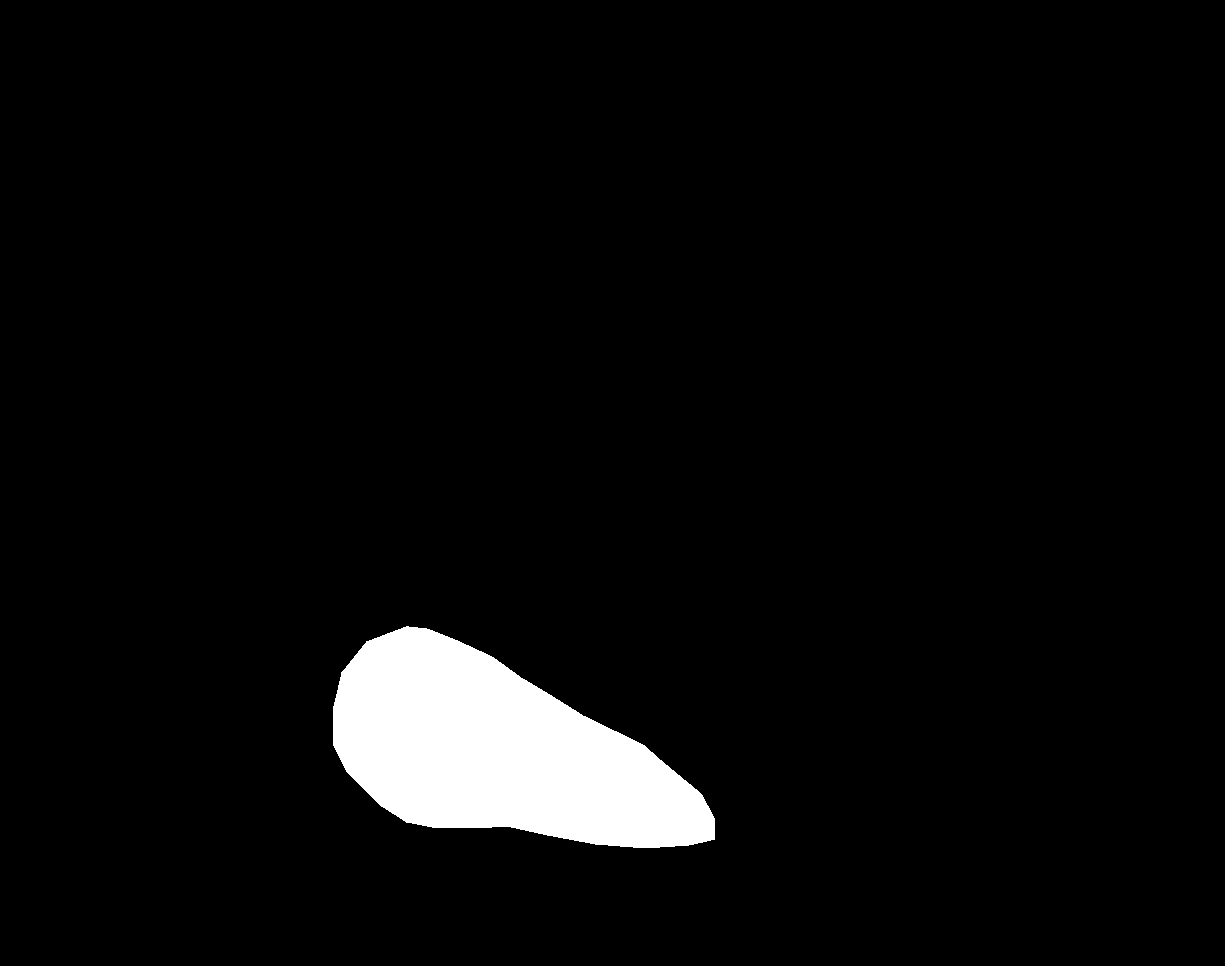}\vspace{-2pt}
\end{subfigure}
\begin{subfigure}{\subw}
\centering
\includegraphics[width=\textwidth]{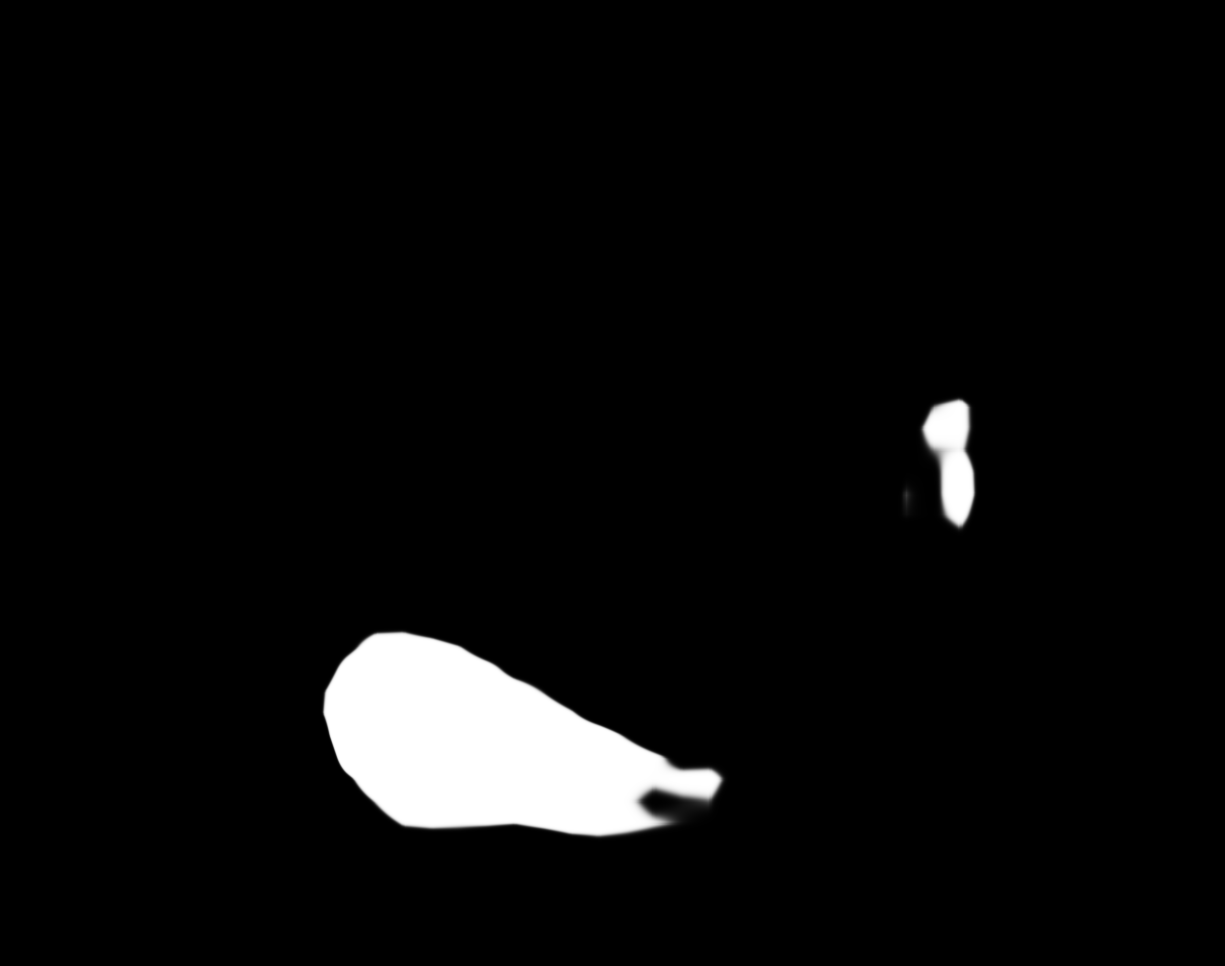}\vspace{-2pt}
\end{subfigure}
\begin{subfigure}{\subw}
\centering
\includegraphics[width=\textwidth]{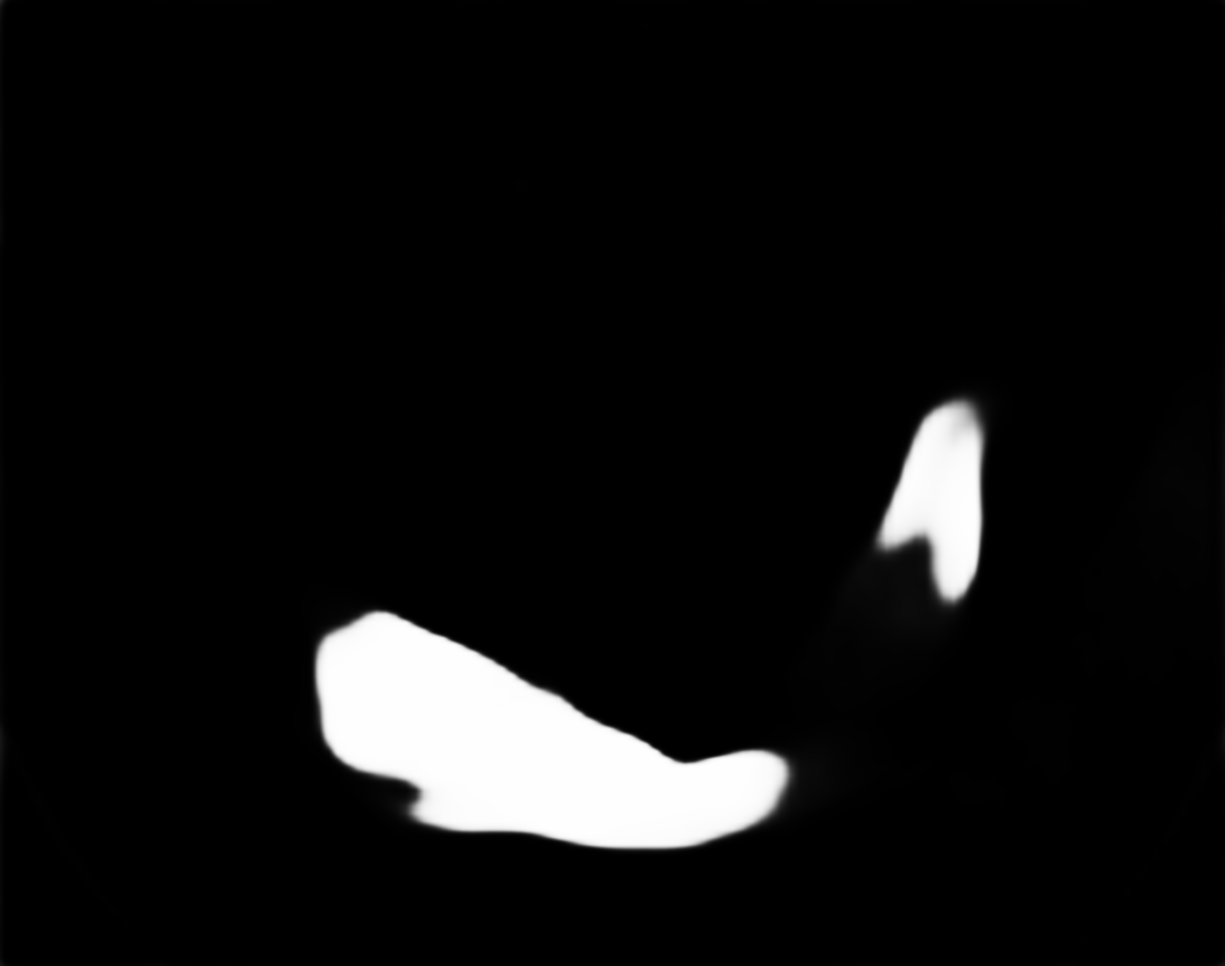}\vspace{-2pt}
\end{subfigure}
\begin{subfigure}{\subw}
\centering
\includegraphics[width=\textwidth]{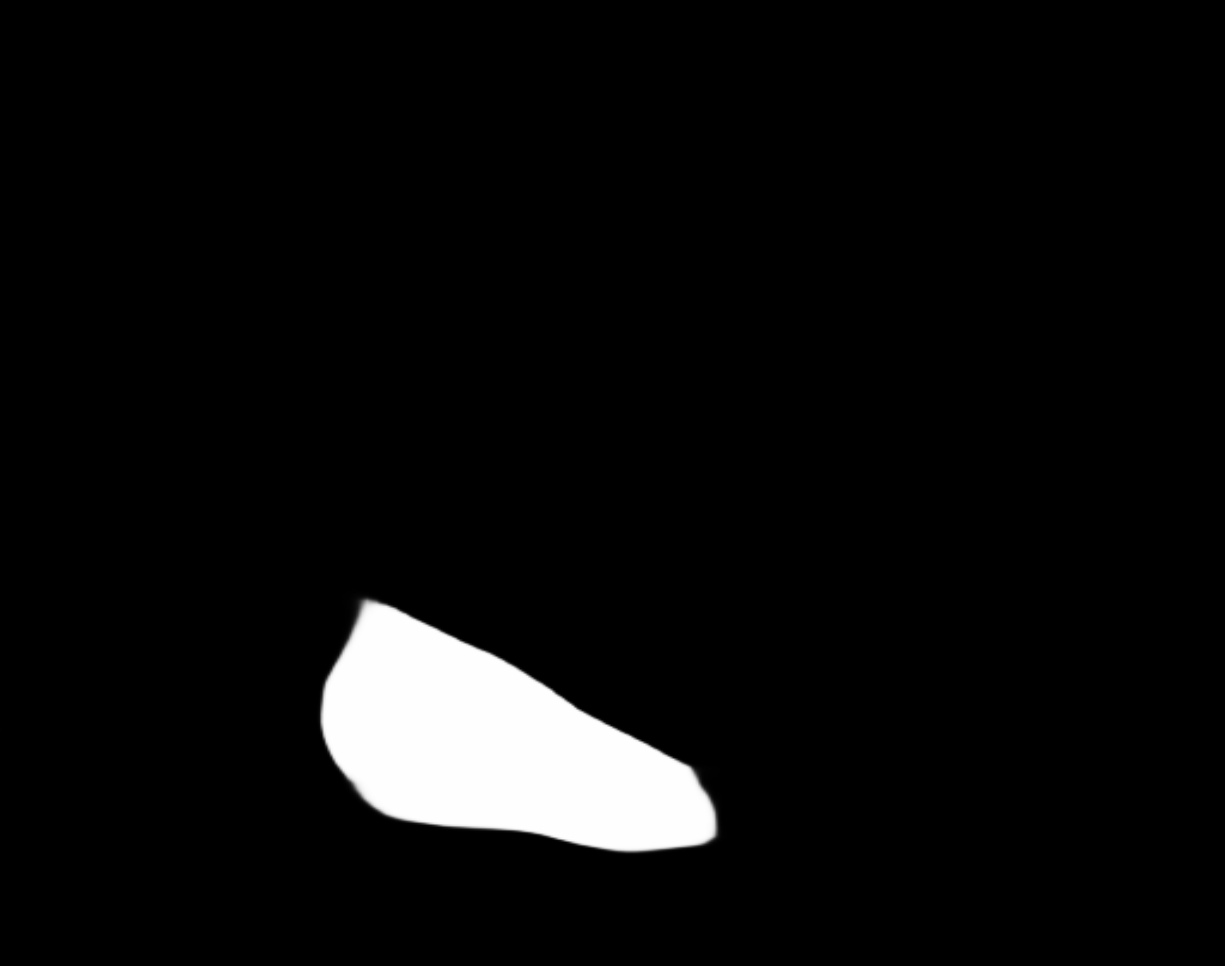}\vspace{-2pt}
\end{subfigure}
\begin{subfigure}{\subw}
\centering
\includegraphics[width=\textwidth]{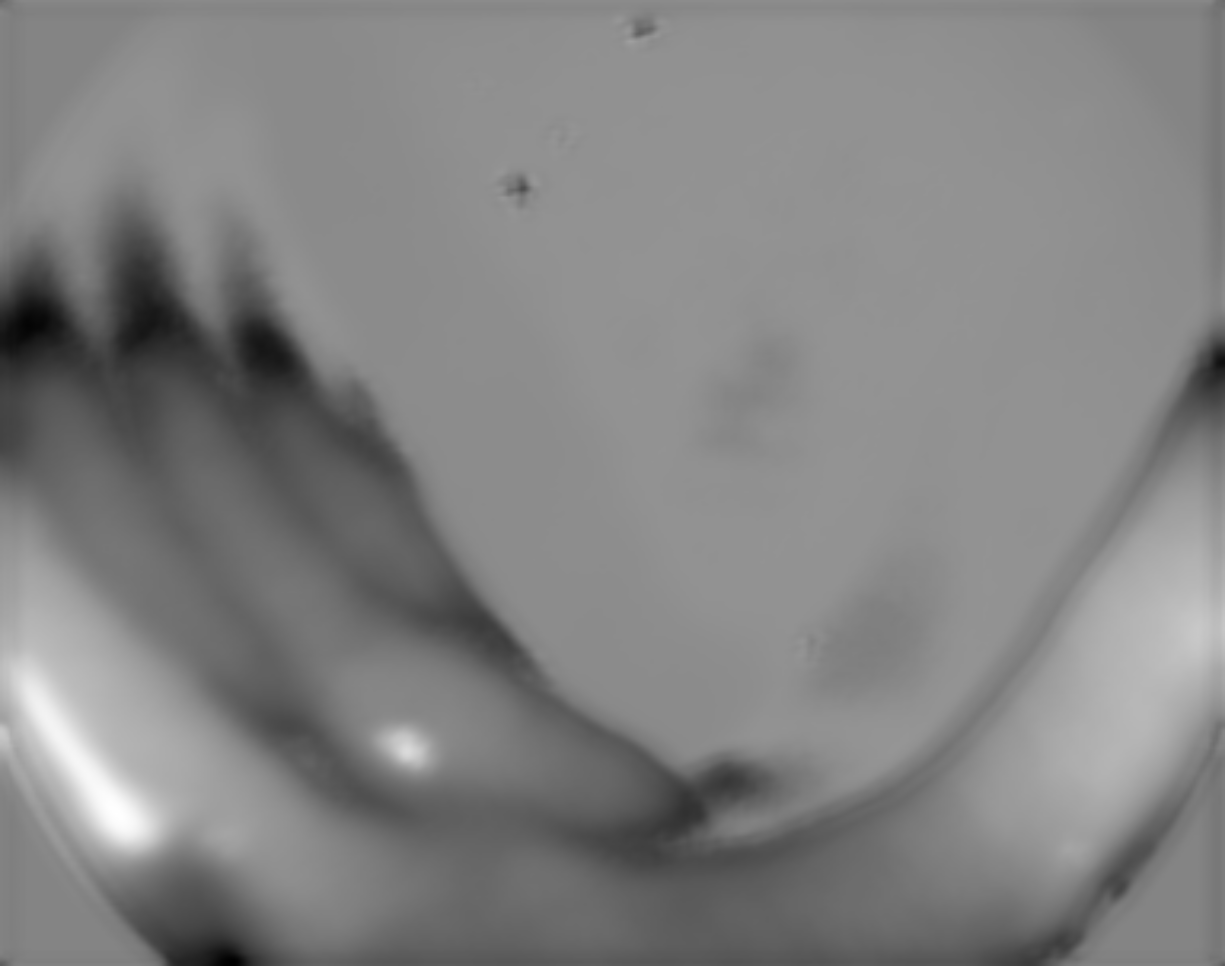}\vspace{-2pt}
\end{subfigure}
\begin{subfigure}{\subw}
\centering
\includegraphics[width=\textwidth]{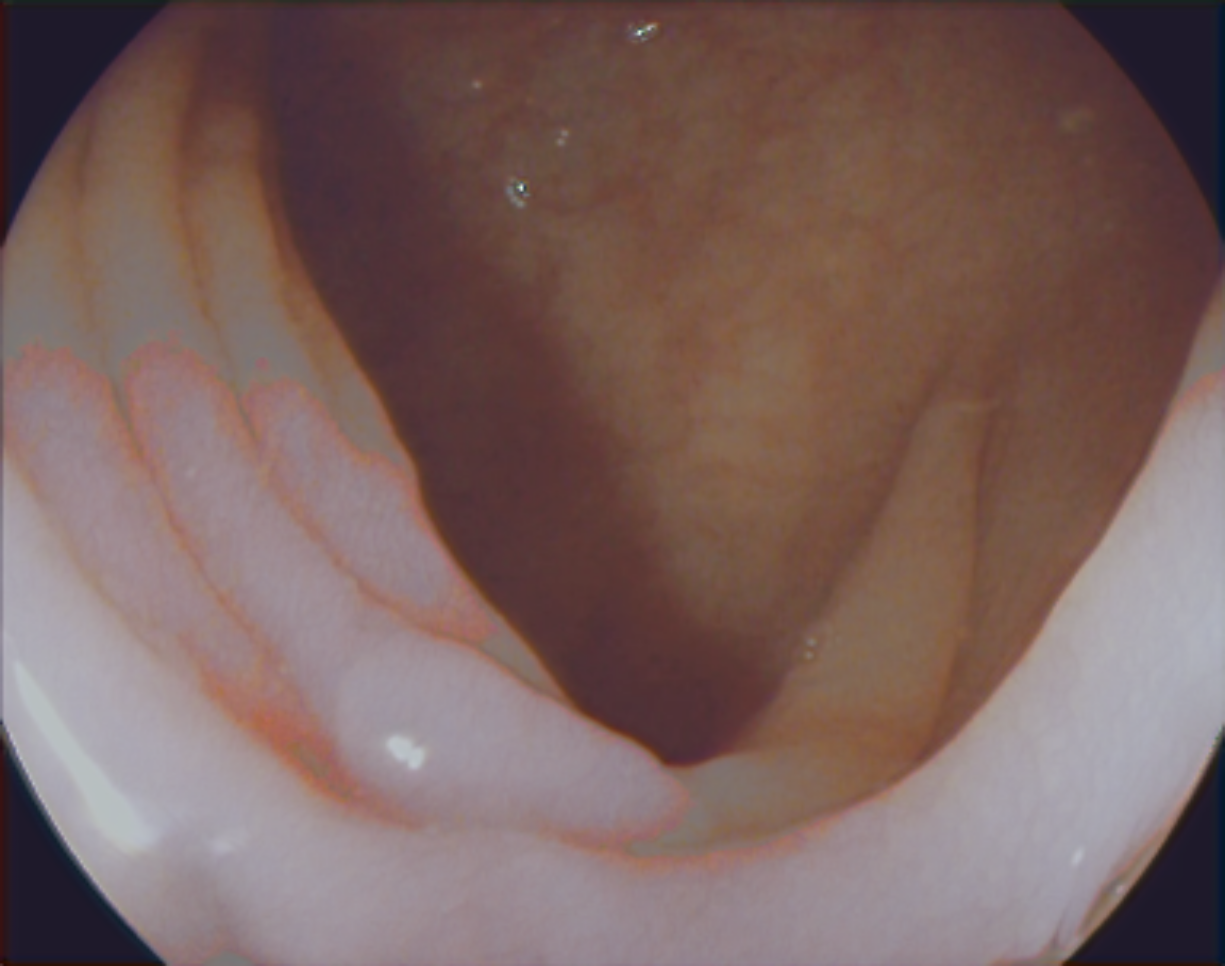}\vspace{-2pt}
\end{subfigure}
\begin{subfigure}{\subw}
\centering
\includegraphics[width=\textwidth]{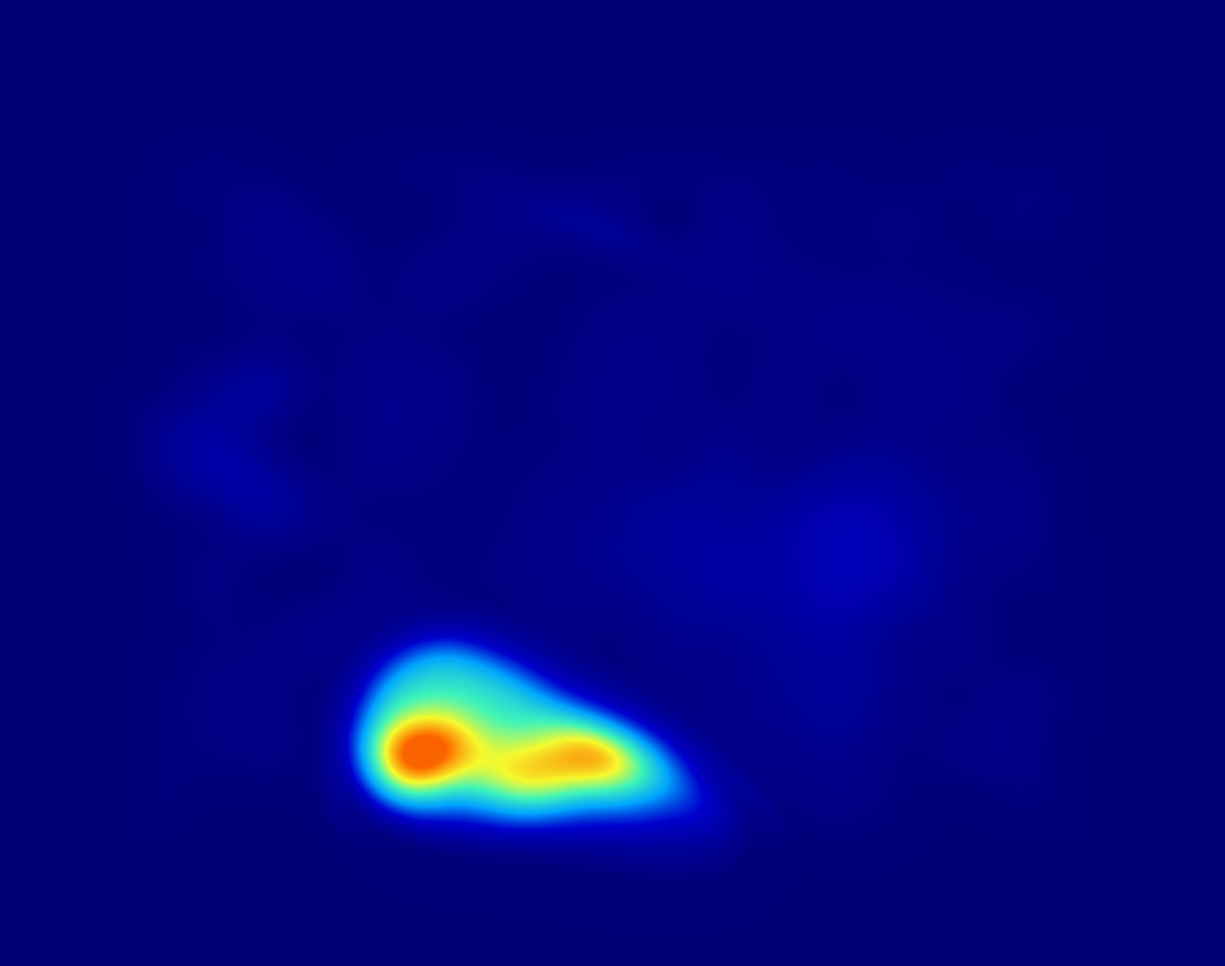}\vspace{-2pt}
\end{subfigure}\\ \vspace{1mm}
\begin{subfigure}{\subw}
\centering
\includegraphics[width=\textwidth]{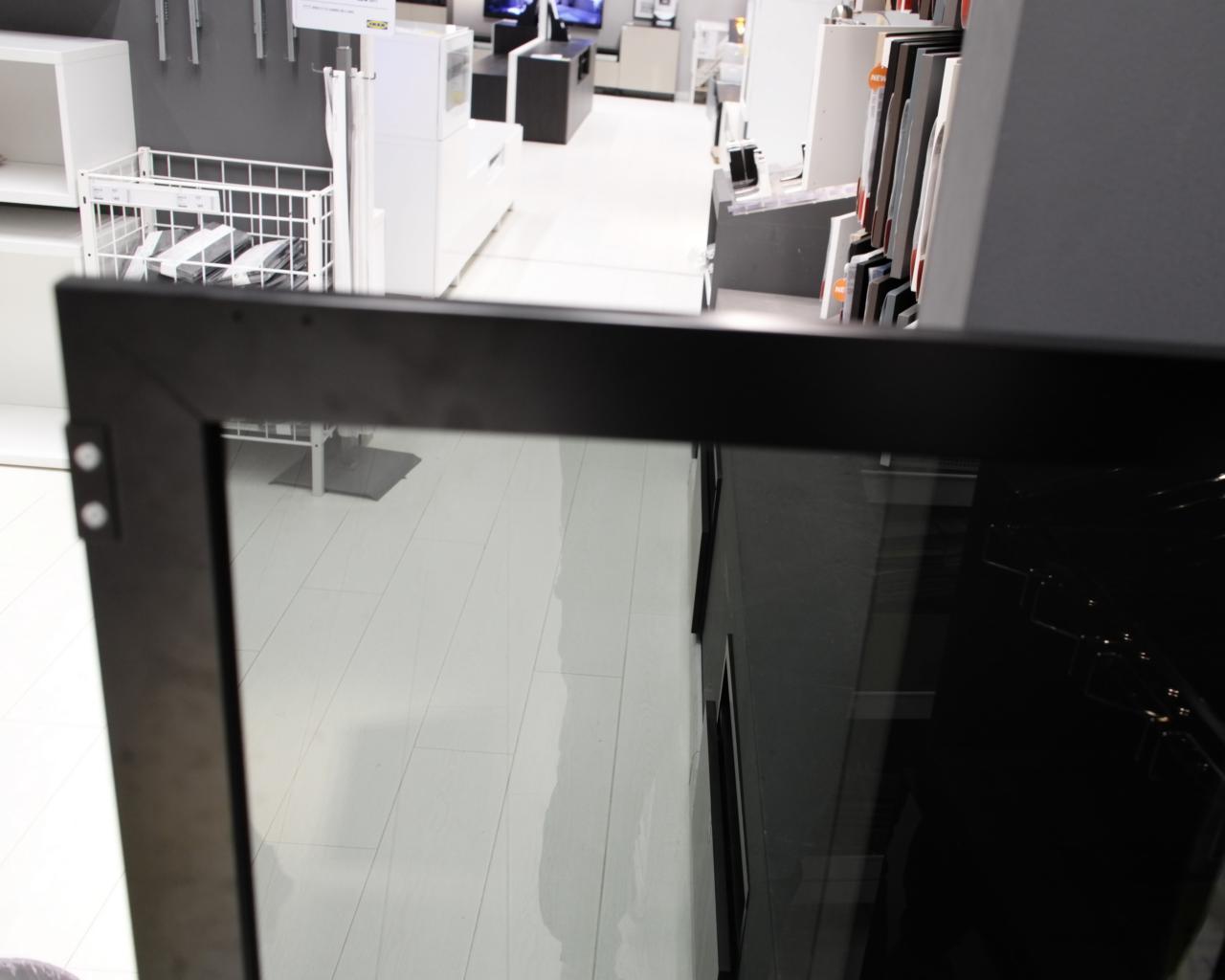}\vspace{-2pt}
\end{subfigure}
\begin{subfigure}{\subw}
\centering
\includegraphics[width=\textwidth]{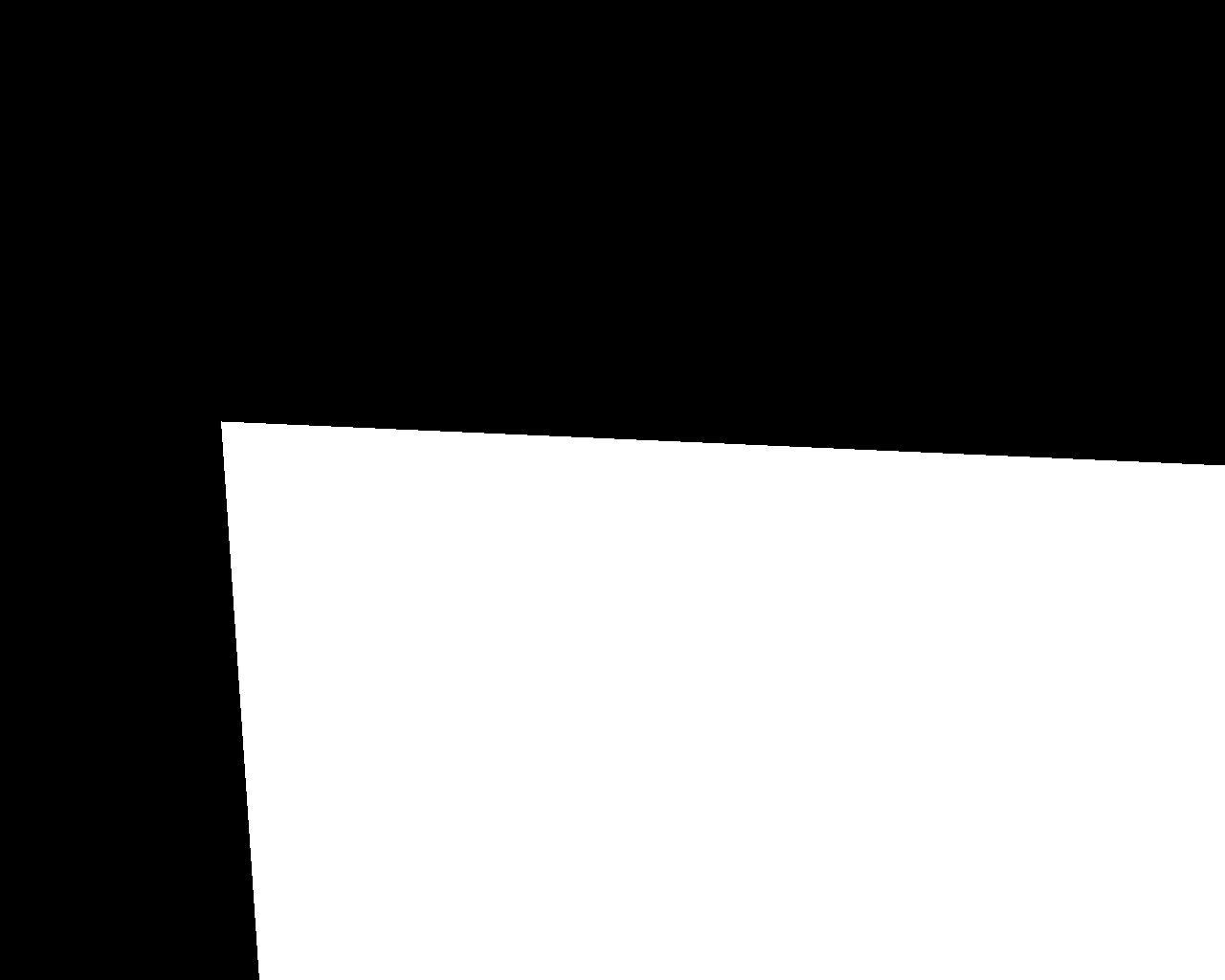}\vspace{-2pt}
\end{subfigure}
\begin{subfigure}{\subw}
\centering
\includegraphics[width=\textwidth]{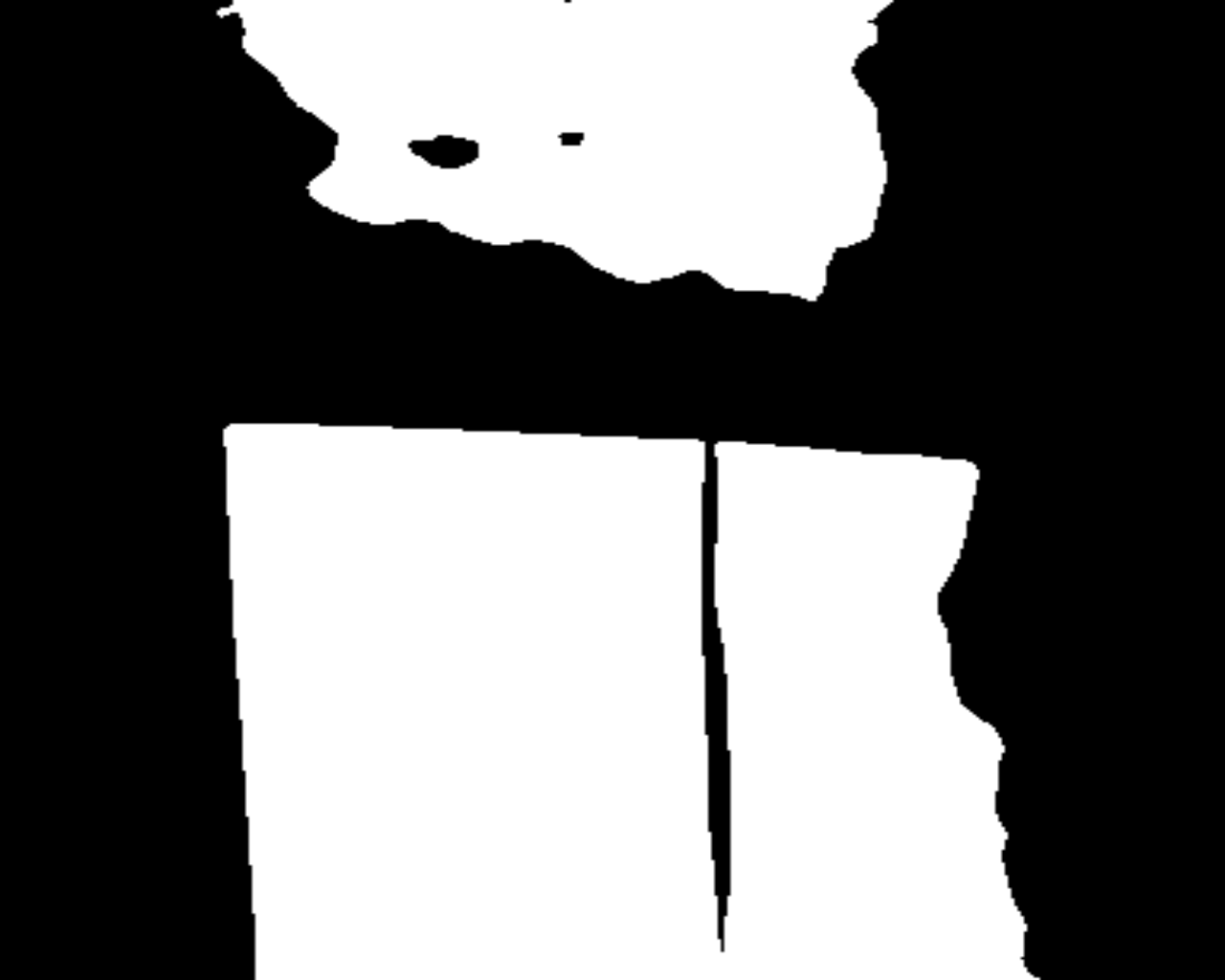}\vspace{-2pt}
\end{subfigure}
\begin{subfigure}{\subw}
\centering
\includegraphics[width=\textwidth]{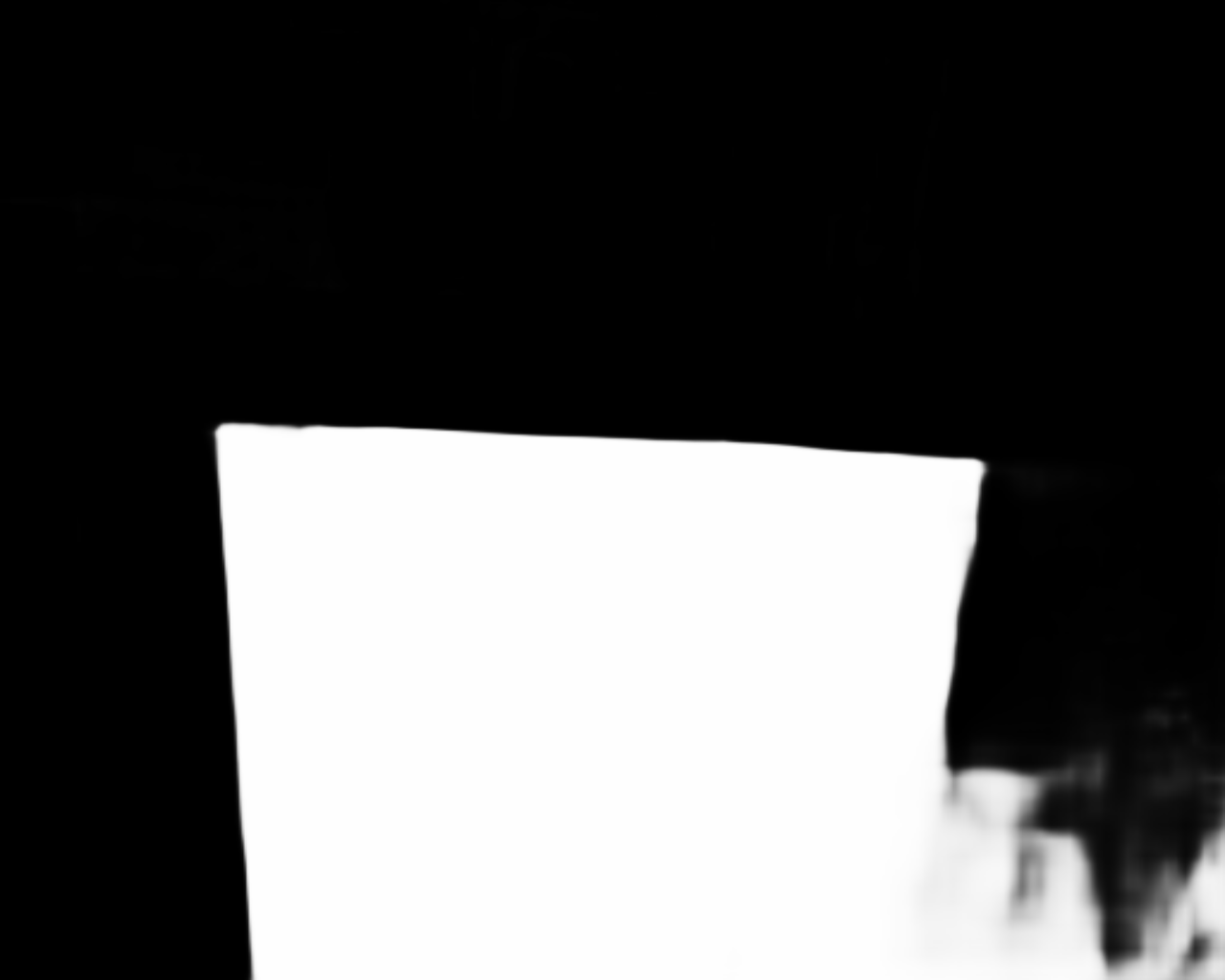}\vspace{-2pt}
\end{subfigure}
\begin{subfigure}{\subw}
\centering
\includegraphics[width=\textwidth]{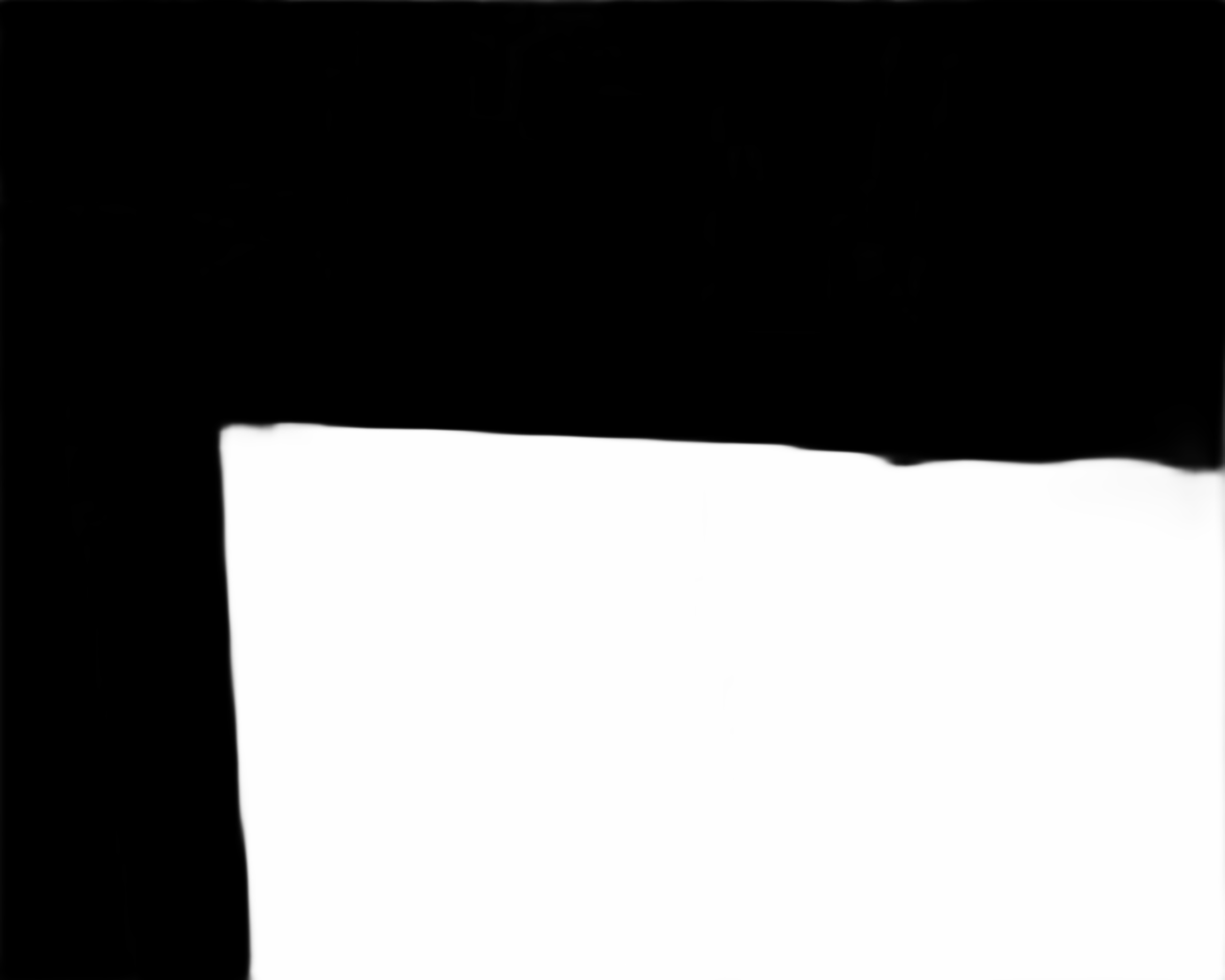}\vspace{-2pt}
\end{subfigure}
\begin{subfigure}{\subw}
\centering
\includegraphics[width=\textwidth]{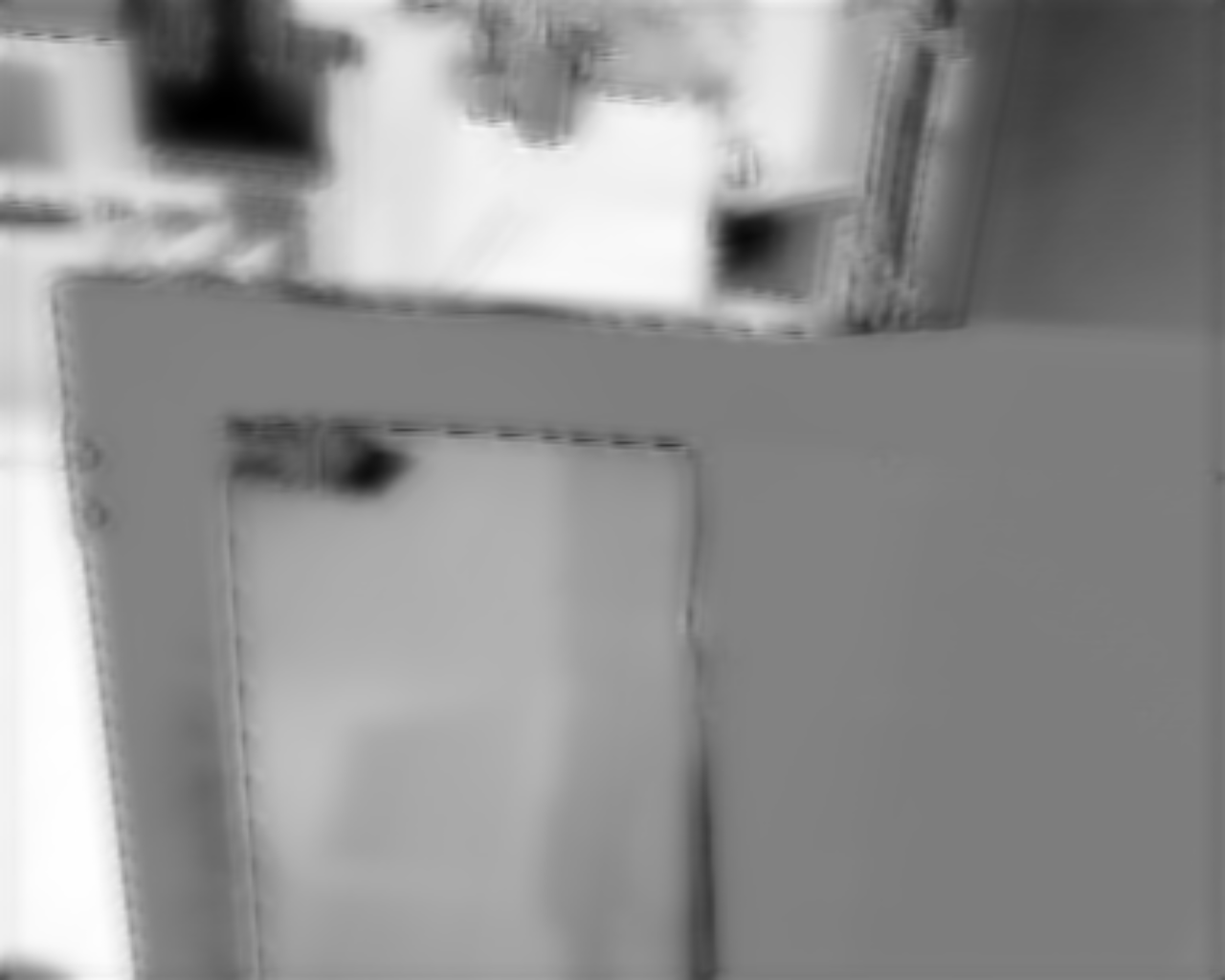}\vspace{-2pt}
\end{subfigure}
\begin{subfigure}{\subw}
\centering
\includegraphics[width=\textwidth]{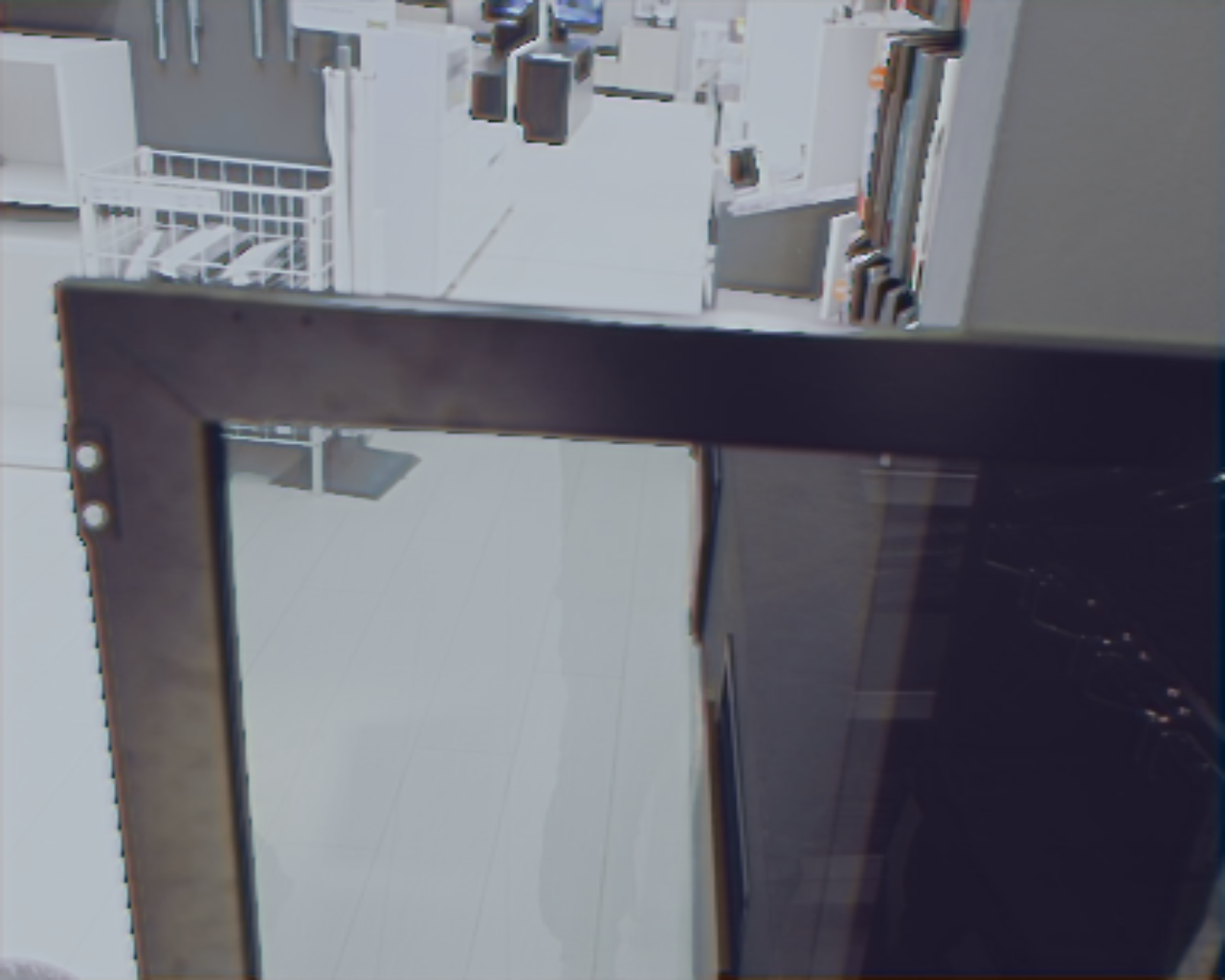}\vspace{-2pt}
\end{subfigure}
\begin{subfigure}{\subw}
\centering
\includegraphics[width=\textwidth]{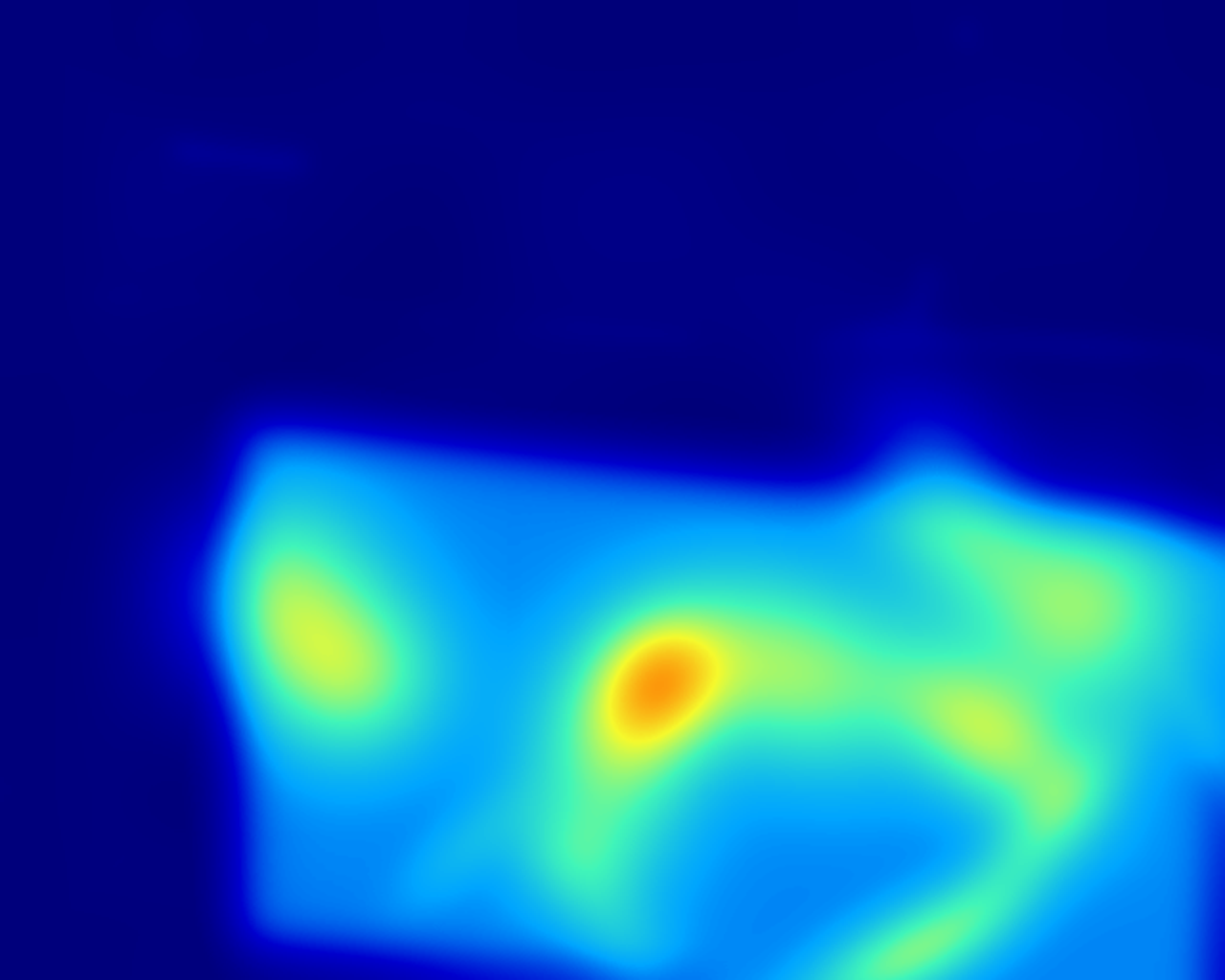}\vspace{-2pt}
\end{subfigure}\\ \vspace{1mm}
\begin{subfigure}{\subw}
\centering
\includegraphics[width=\textwidth]{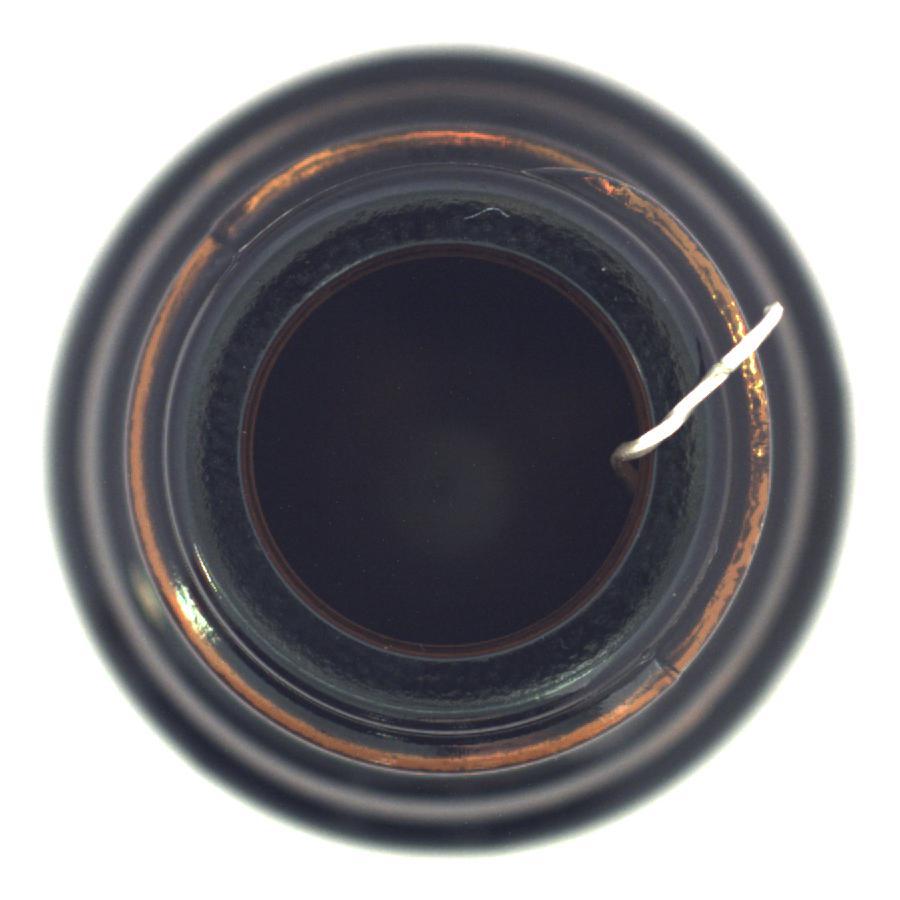}\vspace{-2pt}\caption*{Origin}
\end{subfigure}
\begin{subfigure}{\subw}
\centering
\includegraphics[width=\textwidth]{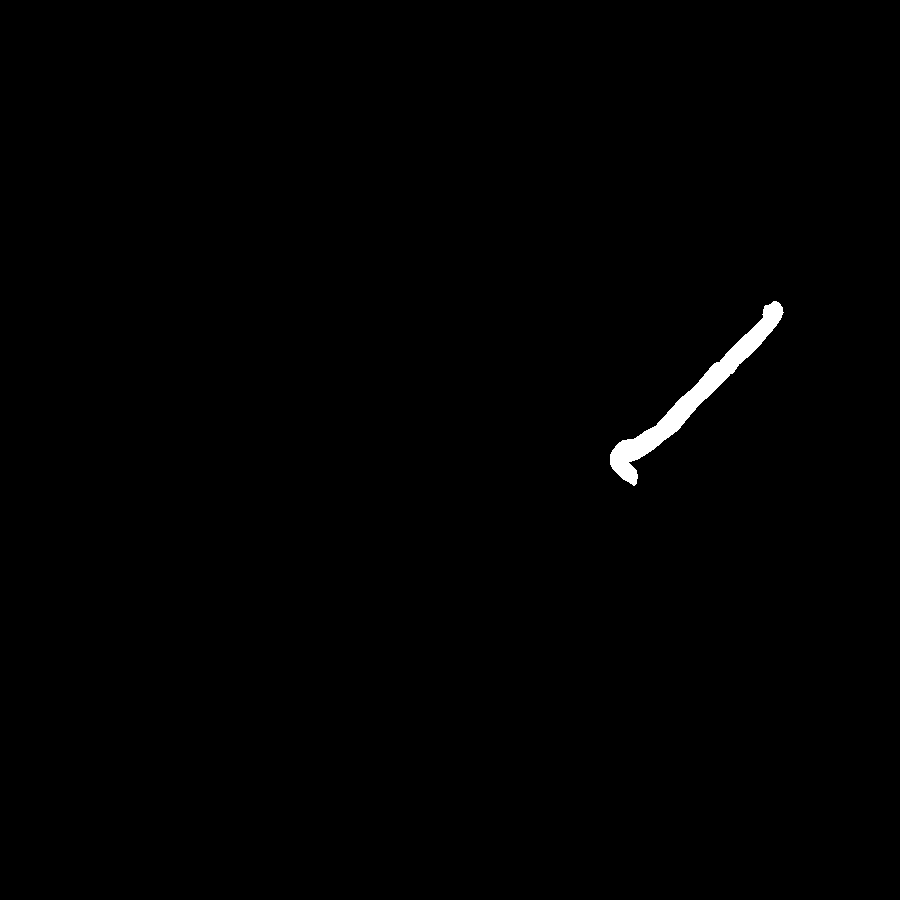}\vspace{-2pt}\caption*{GT}
\end{subfigure}
\begin{subfigure}{\subw}
\centering
\includegraphics[width=\textwidth]{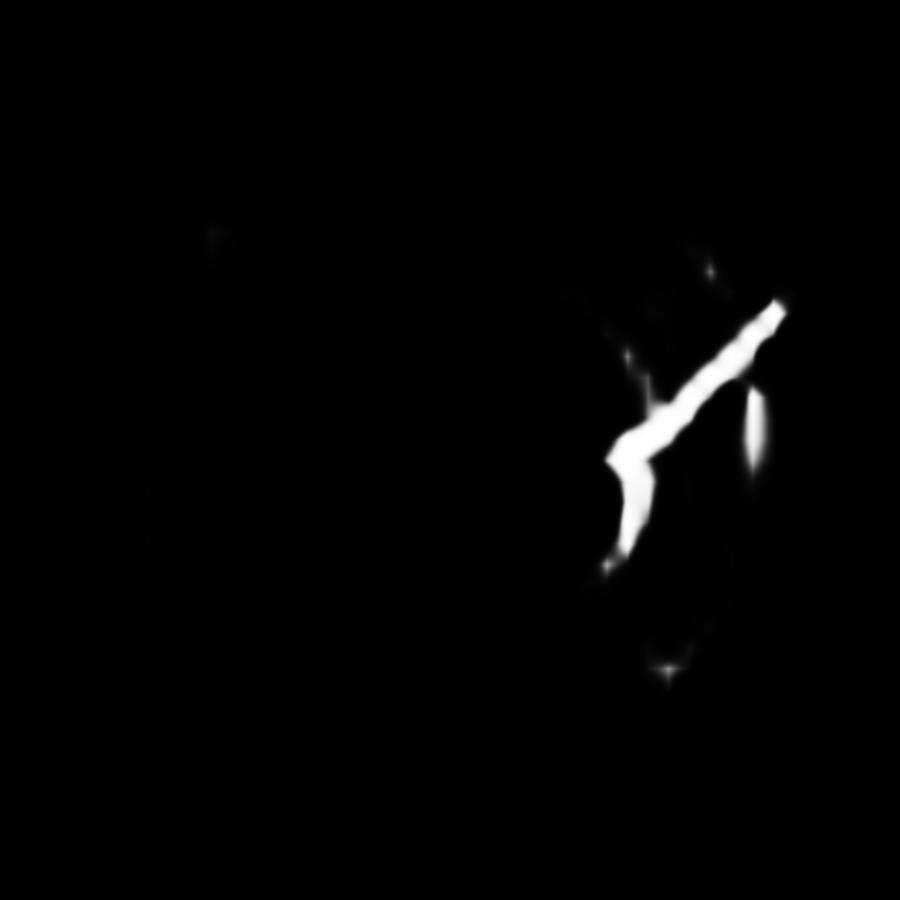}\vspace{-2pt}\caption*{TOP2}
\end{subfigure}
\begin{subfigure}{\subw}
\centering
\includegraphics[width=\textwidth]{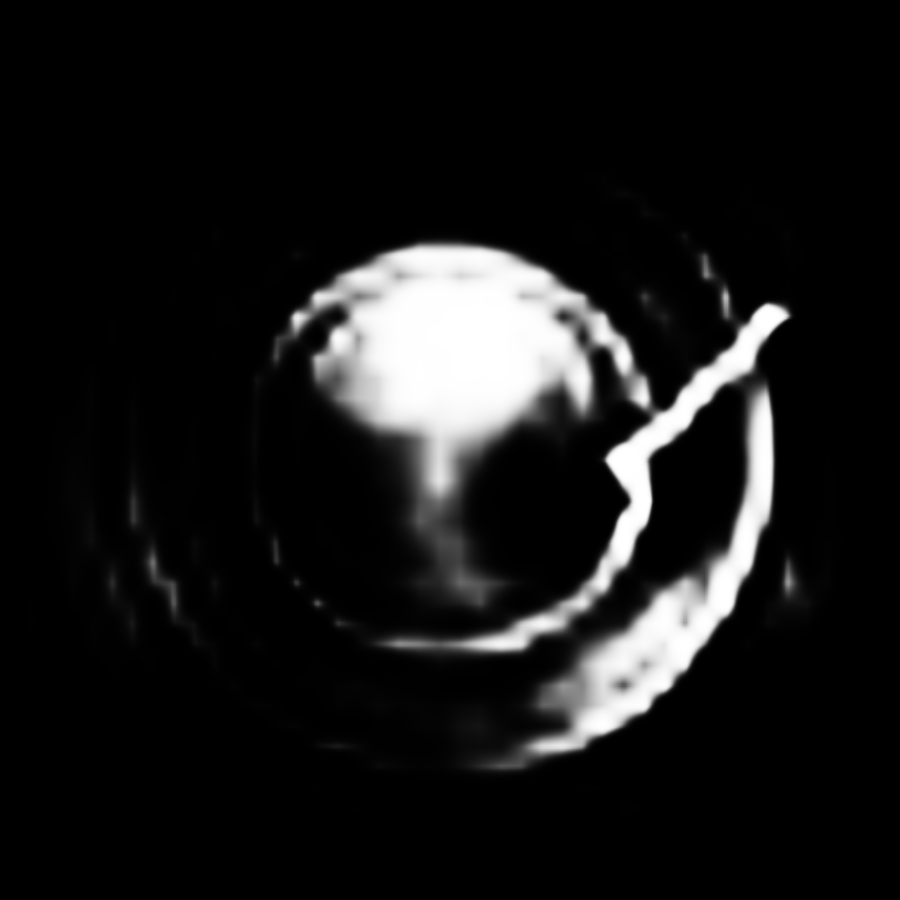}\vspace{-2pt}\caption*{TOP1}
\end{subfigure}
\begin{subfigure}{\subw}
\centering
\includegraphics[width=\textwidth]{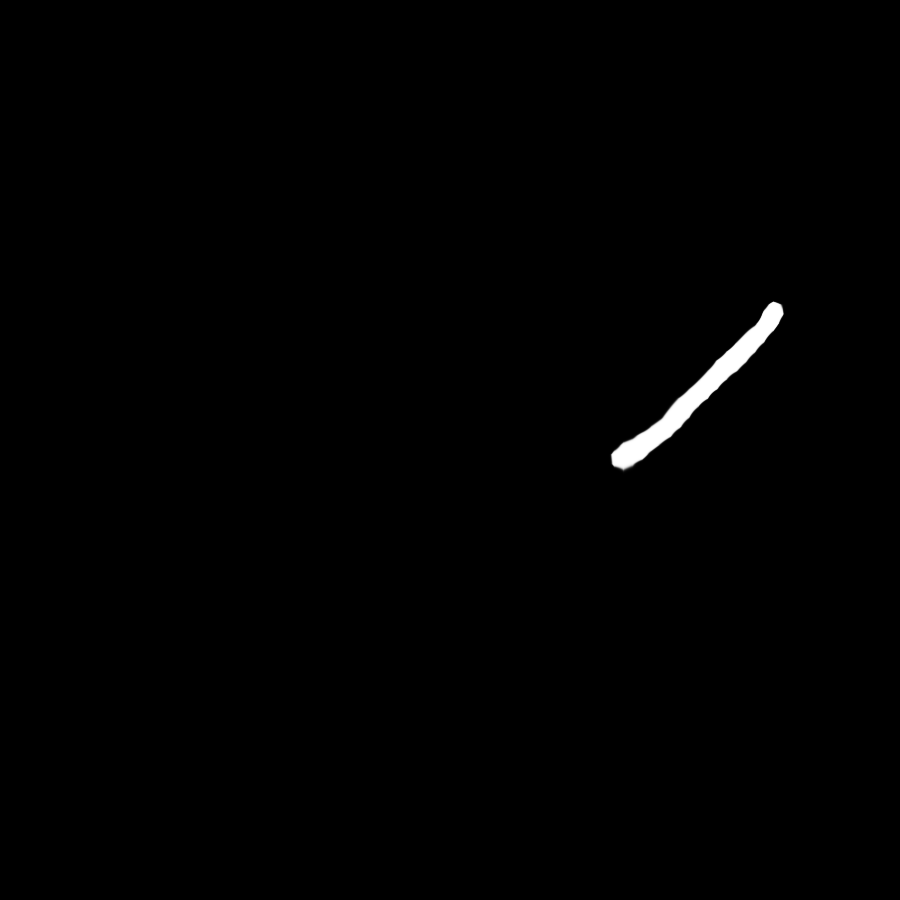}\vspace{-2pt}\caption*{RIDE}
\end{subfigure}
\begin{subfigure}{\subw}
\centering
\includegraphics[width=\textwidth]{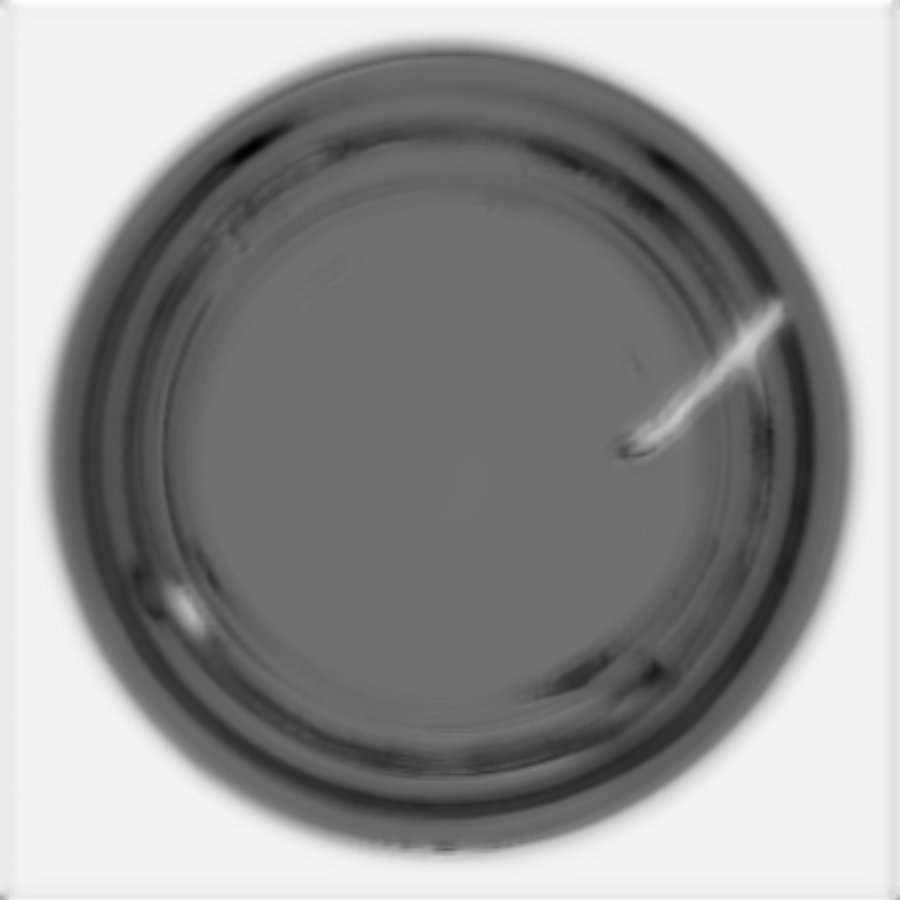}\vspace{-2pt}\caption*{$L$ map}
\end{subfigure}
\begin{subfigure}{\subw}
\centering
\includegraphics[width=\textwidth]{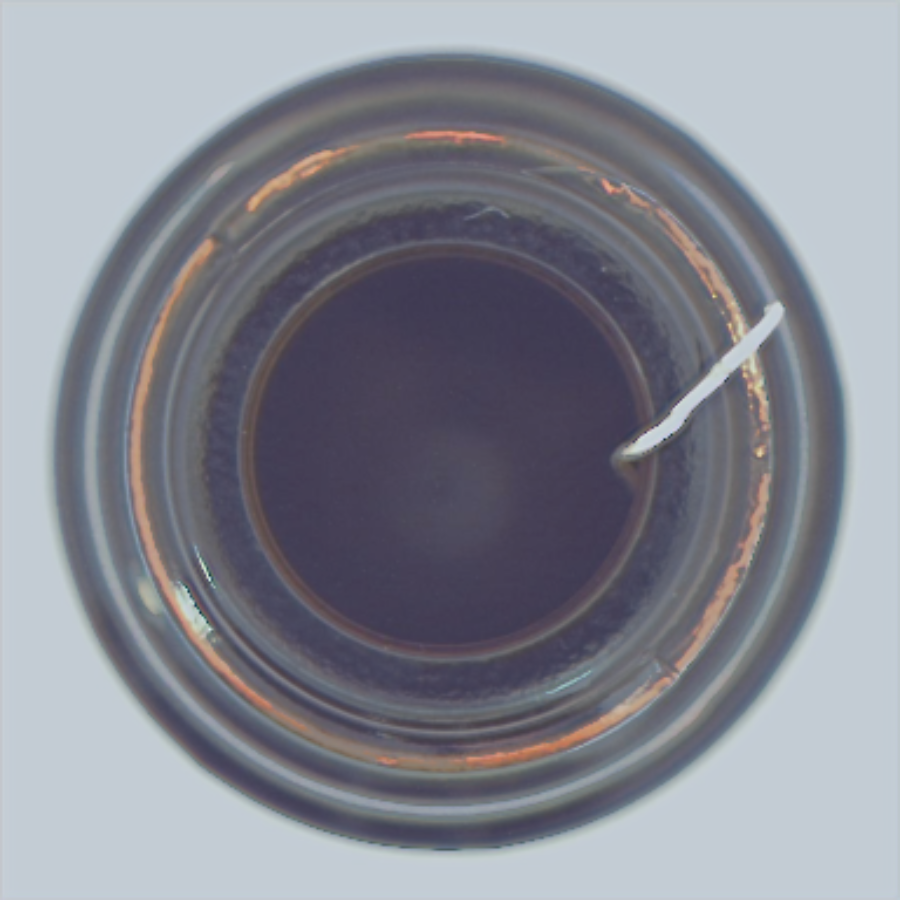}\vspace{-2pt}\caption*{$R$ map}
\end{subfigure}
\begin{subfigure}{\subw}
\centering
\includegraphics[width=\textwidth]{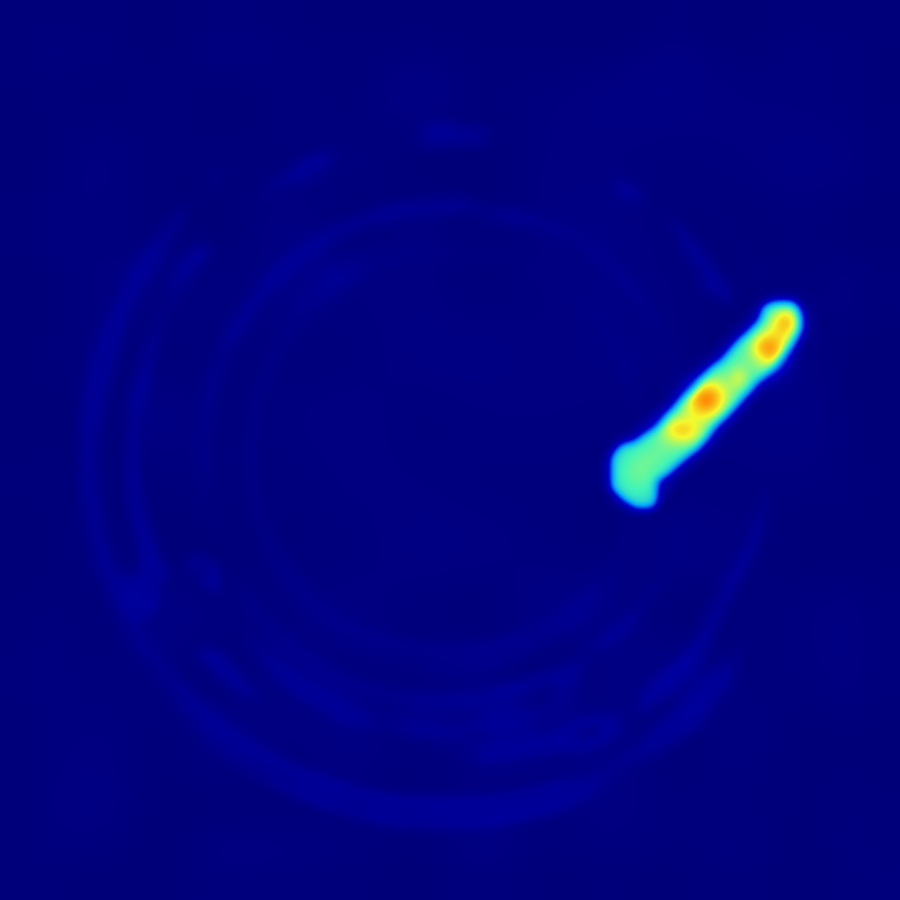}\vspace{-2pt}\caption*{$\Delta_R$}
\end{subfigure}
\caption{Qualitative comparison. We also present learned Retinex decomposition ($L$ and $R$ maps) and discriminability gap map $\Delta_R$. The TOP1/TOP2 baselines are same with those in \cref{table:CODQuanti,table:efficiency,table:PISQuanti,table:TODQuanti,table:CDDQuanti}.}
\label{fig:qualitative}
\vspace{-2mm}
\end{minipage}
\end{figure*}
    
\begin{table*}[tbp!]
\begin{minipage}{1\textwidth}
\setlength{\abovecaptionskip}{0cm} 
\centering
\caption{Efficiency comparison with cutting-edge COD methods. All numbers are obtained under a matched evaluation protocol (input size $352 \times 352$, FP32, batch size 1, single RTX 4090, synchronized CUDA timing) using each method's official released code.} \label{table:efficiency}
\resizebox{\columnwidth}{!}{
\setlength{\tabcolsep}{1.6mm}
\begin{tabular}{l|ccc|ccc|ccc}
\toprule
\multirow{2}{*}{Efficiency} & \multicolumn{3}{c|}{ResNet50} & \multicolumn{3}{c|}{Res2Net50} & \multicolumn{3}{c}{PVT V2} \\ \cline{2-10}
& FocusDiff~\cite{zhao2025focusdiffuser} & RUN~\cite{he2025run}    & \cellcolor{c2!20}  RIDE (Ours)  & FEDER~\cite{He2023Camouflaged}    & RUN~\cite{he2025run}      &\cellcolor{c2!20}  RIDE (Ours)    & CamoFocus~\cite{khan2024camofocus}  & RUN~\cite{he2025run}   &\cellcolor{c2!20}  RIDE (Ours)  \\ \midrule
Parameters (M) $\downarrow$             & 166.17   & 30.41  & \cellcolor{c2!20} \textbf{28.50} & 45.92    & 30.57     & \cellcolor{c2!20}  \textbf{28.66} & 68.85     & 65.17 &   \cellcolor{c2!20} \textbf{63.25}   \\
FLOPs (G) $\downarrow$                  & 7618.49  &    43.36  & \cellcolor{c2!20} \textbf{35.18}   & 50.03    & 45.73    & \cellcolor{c2!20} \textbf{35.38}    & 91.35       & 58.53 & \cellcolor{c2!20}  \textbf{49.18}  \\
FPS $\uparrow$                        & 0.23     & 22.75  & \cellcolor{c2!20}  \textbf{27.20}   & 14.02    & 20.26    & \cellcolor{c2!20} \textbf{25.40}    & 9.63       & 15.82 & \cellcolor{c2!20} \textbf{18.26}  \\ \bottomrule
\end{tabular}}
\end{minipage} 
\vspace{-3mm}
\end{table*}

\noindent\textbf{Camouflaged object detection.} \cref{table:CODQuanti,fig:qualitative} present the comprehensive comparison. \methodname{} achieves the best performance across all metrics. Notably, \methodname{} demonstrates large improvements in $F_\beta$, which measures the balance between precision and recall. This suggests that Retinex decomposition provides both better object localization (recall) and cleaner boundary delineation (precision). As shown in Table~\ref{table:efficiency}, RIDE also achieves the best efficiency among recent methods. 
To address potential single-run variance concerns, we additionally report a 5-seed analysis on COD10K and NC4K (ResNet-50) for \methodname{} and the strongest competing baselines (RUN, FSEL) in Appendix~\ref{app:multiseed}.

\noindent\textbf{Polyp image segmentation.} Polyp endoscopy is a strong test of our theory: point-source illumination on curved mucosal surfaces produces precisely the $\rho < 0$ regime predicted to maximize Retinex gains. As shown in \cref{table:PISQuanti,fig:qualitative}, \methodname{} achieves the best results, verifying our superiority.

\begin{table*}[tbp!]
\begin{minipage}{0.317\textwidth}
\setlength{\abovecaptionskip}{0cm} 
		\setlength{\belowcaptionskip}{0cm}
\centering
\caption{Results on the PIS task (\textit{ETIS}~\cite{silva2014toward}).
        } \label{table:PISQuanti}
	\resizebox{1\columnwidth}{!}{
		\setlength{\tabcolsep}{1.67mm}
	\begin{tabular}{l|ccc}
		\toprule 
		{Methods} 
		& \multicolumn{1}{c}{\cellcolor{gray!40}mDice~$\uparrow$} & \multicolumn{1}{c}{\cellcolor{gray!40}mIoU~$\uparrow$} & \multicolumn{1}{c}{\cellcolor{gray!40}$S_\alpha$~$\uparrow$} \\ \midrule
        PolypPVT~\cite{dong2023polyp}  & {{0.787}} & {{0.706}} & {{0.871}} \\
        CoInNet~\cite{jain2023coinnet} & 0.759 & 0.690 & 0.859 \\
        LSSNet~\cite{wang2024lssnet} & 0.779 & 0.701 & 0.867 \\
        MEGANet~\cite{bui2024meganet} & 0.739 & 0.665 & 0.836 \\
        RUN~\cite{he2025run}  & {{0.788}} &{{0.709}} & {{0.878}}   \\
        \rowcolor{c2!20} RIDE (Ours) & \textbf{0.793} & \textbf{0.718} & \textbf{0.884} \\
	 \bottomrule                      
	\end{tabular}}    
\end{minipage}
\begin{minipage}{0.3\textwidth}
\setlength{\abovecaptionskip}{0cm} 
		\setlength{\belowcaptionskip}{0cm}
	\centering
        \caption{ Results on the TOD task (\textit{GDD}~\cite{mei2020don}). 
        } \label{table:TODQuanti}
	\resizebox{\columnwidth}{!}{
		\setlength{\tabcolsep}{1.7mm}
	\begin{tabular}{l|ccc}\toprule 

		 {Methods} & \cellcolor{gray!40}mIoU~$\uparrow$&\cellcolor{gray!40}$F_\beta^{max}$~$\uparrow$&\cellcolor{gray!40} $M$~$\downarrow$ \\ \midrule
        EBLNet~\cite{he2021enhanced} & 0.870                                 & 0.922                                 & 0.064   \\
        RFENet~\cite{fan2023rfenet} & 0.886 & 0.938 & 0.057  \\
        IEBAF~\cite{han2023internal} & 0.887 & {{0.944}} & 0.056  \\
        GhostNet~\cite{yan2024ghostingnet} & {{0.893}}  & 0.943 & {{0.054}} \\
	RUN~\cite{he2025run}  & {{0.895}} & {{0.952}} & {{0.051}}  \\ 
    \rowcolor{c2!20} RIDE (Ours) & \textbf{0.907} & \textbf{0.955} & \textbf{0.047}\\
    \bottomrule  \end{tabular}}
\end{minipage}
\begin{minipage}{0.363\textwidth}
\centering
	\setlength{\abovecaptionskip}{0cm} 
		\setlength{\belowcaptionskip}{0cm}
	\caption{Results on the CDD task (\textit{CDS2K}~\cite{fan2023advances}).  
    } \label{table:CDDQuanti}
	\resizebox{1\columnwidth}{!}{
		\setlength{\tabcolsep}{2mm}
		\begin{tabular}{l|cccccccc}
        \toprule 
Methods    & {\cellcolor{gray!40}$S_\alpha$~$\uparrow$} & {\cellcolor{gray!40}$M$~$\downarrow$} & {\cellcolor{gray!40}$E_\phi$~$\uparrow$} & {\cellcolor{gray!40}$F_\beta$~$\uparrow$} \\ \midrule
HitNet~\cite{hu2022high}     & 0.563 & 0.118 & 0.564  & {{0.298}}   \\
CamoFormer~\cite{yin2024camoformer} & {{0.589}} & {{0.100}} & {{0.588}}  & {{0.330}} \\
OAFormer~\cite{yang2023oaformer}   & 0.541 & 0.121 & 0.535  & 0.216  \\
FEDER~\cite{He2023Camouflaged} & 0.538 & 0.070 & 0.586 &  0.288  \\
RUN~\cite{he2025run} & {{0.590}} &  {{0.068}} & {{0.595}}  & {{0.298}}  \\
\rowcolor{c2!20} RIDE (Ours) & \textbf{0.609} & \textbf{0.061} & \textbf{0.612} & \textbf{0.334} \\
\bottomrule
\end{tabular}}
\end{minipage}\\
\begin{minipage}{\textwidth}
\setlength{\abovecaptionskip}{0cm} 
\setlength{\belowcaptionskip}{0cm}
\centering
\caption{{Generalization beyond COS.} ``+R'' denotes augmentation with our Retinex strategy.} \label{tab:generalization}
\resizebox{\columnwidth}{!}{
\setlength{\tabcolsep}{2.5mm}
\begin{tabular}{l|l|c|cc|c}
\toprule
Tasks / Datasets & Base methods & Metrics & Baseline & +R (Ours) & $\Delta (\%)$ \\
\midrule
Semantic Segmentation (ADE20K~\cite{zhou2017scene}) & PEM~\cite{cavagnero2024pem} & mIoU $\uparrow$ & 45.0 & 46.0 & 2.22 \\
Instance Segmentation (COCO~\cite{lin2014microsoft}) & Mask2Former~\cite{cheng2022masked} & AP $\uparrow$ & 38.0 & 38.6 & 1.58 \\
Shadow Detection (SBU~\cite{vicente2016large} ) & Spider~\cite{zhao2024spider} & BER $\downarrow$ & 0.040 & 0.038 & 5.00 \\
Infrared Small Target Detection (IRSTD-1k~\cite{zhang2022isnet}) & IRSAM~\cite{zhang2024irsam} & IoU $\uparrow$ & 73.69 & 75.53 & 2.50 \\
Salient Object Detection (DUTS-test~\cite{wang2017learning}) & RUN~\cite{he2025run} & $F_\beta$ $\uparrow$ & 0.886 & 0.901 & 1.69 \\
Video COS (CAD~\cite{cheng2022implicit}) & ZoomNeXt~\cite{pang2024zoomnext} & $S_\alpha$ $\uparrow$ & 0.757 & 0.769 & 1.59 \\
\bottomrule
\end{tabular}
} 
\end{minipage}
\vspace{-4mm}
\end{table*} 
\begin{table}[t]
\begin{minipage}{0.515\textwidth}
\setlength{\abovecaptionskip}{0cm} 
\setlength{\belowcaptionskip}{0cm}
\centering
\caption{Ablation study on \textit{COD10K}. For 6-9, we replace TRD with other decomposition methods.
} \label{tab:ablation}
\resizebox{\columnwidth}{!}{
\setlength{\tabcolsep}{1.2mm}
\begin{tabular}{c|ccccc|cccc}
\toprule
\# & TRD & DGA & CBC & ME & RB & $M \downarrow$ & $F_\beta \uparrow$ & $E_\phi \uparrow$ & $S_\alpha \uparrow$ \\
\midrule
1 & \cmark & & & & & 0.032 & 0.740 & 0.897 & 0.819 \\
2 & \cmark & \cmark & & & & 0.030 & 0.749 & 0.902 & 0.825 \\
3 & \cmark & \cmark & \cmark & & & 0.029 & 0.755 & 0.906 & 0.829 \\
4 & \cmark & \cmark & \cmark & \cmark & & 0.028 & 0.760 & 0.910 & 0.832 \\
5 & \cmark & \cmark & \cmark & \cmark & \cmark & \textbf{0.028} & \textbf{0.763} & \textbf{0.912} & \textbf{0.834} \\
\midrule
6 & FFT & \cmark & \cmark & -- & -- & 0.032 & 0.738 & 0.895 & 0.818 \\
7 & Wavelet & \cmark & \cmark & -- & -- & 0.031 & 0.742 & 0.899 & 0.821 \\
8 & DCT & \cmark & \cmark & -- & -- & 0.032 & 0.736 & 0.894 & 0.817 \\
9 & Fixed Reti. & \cmark & \cmark & -- & -- & 0.031 & 0.745 & 0.901 & 0.823 \\
\bottomrule
\end{tabular}}
\end{minipage}\hfill
\begin{minipage}{0.473\textwidth}
\setlength{\abovecaptionskip}{0cm} 
\setlength{\belowcaptionskip}{0cm}
\centering
\caption{DGA design ablation, including standard fusion rules, gap design, and window size.} \label{tab:ablationDGA}
\resizebox{\columnwidth}{!}{
\setlength{\tabcolsep}{0.95mm}
\begin{tabular}{l|cccc}
\toprule
Fusion strategy & $M\!\downarrow$ & $F_\beta\!\uparrow$ & $E_\phi\!\uparrow$ & $S_\alpha\!\uparrow$ \\
\midrule
Concatenation (along channel) & 0.031 & 0.745 & 0.900 & 0.823 \\
Channel attention (SE) & 0.030 & 0.750 & 0.903 & 0.826 \\
Cross attention (Transformer) & 0.029 & 0.753 & 0.905 & 0.828 \\
\midrule
DGA (variance only, no gap) & 0.030 & 0.751 & 0.904 & 0.827 \\
\midrule
DGA (full, $k = 3$) & 0.029 & 0.757 & 0.908 & 0.831 \\
DGA (full, $k = 5$) & 0.029 & 0.760 & 0.910 & 0.832 \\
DGA (full, $k = 11$) & 0.029 & 0.755 & 0.907 & 0.830 \\
\rowcolor{c2!20} {DGA (full, $k = 7$, ours)} & \textbf{0.028} & \textbf{0.763} & \textbf{0.912} & \textbf{0.834} \\
\bottomrule
\end{tabular}
}
\end{minipage}
\vspace{-8mm}
\end{table}

\begin{table}[t]
\begin{minipage}{0.44\textwidth}
\setlength{\abovecaptionskip}{0cm}
\setlength{\belowcaptionskip}{0cm}
\centering
\caption{Contrastive loss space comparison.}
\label{tab:cl_space}
\resizebox{\columnwidth}{!}{
\setlength{\tabcolsep}{1.5mm}
\begin{tabular}{l|cccc}
\toprule
Contrastive Space & $M\!\downarrow$ & $F_\beta\!\uparrow$ & $E_\phi\!\uparrow$ & $S_\alpha\!\uparrow$ \\
\midrule
None (no $\Lcontrast$) & 0.030 & 0.749 & 0.902 & 0.825 \\
$I$-space & 0.031 & 0.745 & 0.900 & 0.823 \\
$L$-space & 0.029 & 0.752 & 0.904 & 0.827 \\
\rowcolor{c2!20} $R$-space (ours) & \textbf{0.028} & \textbf{0.763} & \textbf{0.912} & \textbf{0.834} \\
($I$ + $R$)-space & 0.029 & 0.757 & 0.908 & 0.830 \\
\bottomrule
\end{tabular}}
\end{minipage}
\hfill
\begin{minipage}{0.555\textwidth}
\setlength{\abovecaptionskip}{0cm}
\setlength{\belowcaptionskip}{0cm}
\centering
\caption{Alternative homogeneous decompositions. 
}
\label{tab:alt_decomp}
\resizebox{\columnwidth}{!}{
\setlength{\tabcolsep}{1.55mm}
\begin{tabular}{l|cccc}
\toprule
Decomposition & $M\!\downarrow$ & $F_\beta\!\uparrow$ & $E_\phi\!\uparrow$ & $S_\alpha\!\uparrow$ \\
\midrule
No decomposition (baseline) & 0.034 & 0.731 & 0.891 & 0.812 \\
Guided Filter (edge-preserving) & 0.033 & 0.735 & 0.894 & 0.815 \\
Color Constancy (Gray World) & 0.034 & 0.733 & 0.893 & 0.813 \\
Bilateral (texture/structure) & 0.033 & 0.737 & 0.895 & 0.816 \\
\rowcolor{c2!20} {Retinex (ours)} & \textbf{0.028} & \textbf{0.763} & \textbf{0.912} & \textbf{0.834} \\
\bottomrule
\end{tabular}}
\end{minipage}
\vspace{-2mm}
\end{table}

\noindent\textbf{Transparent object detection.} Glass surfaces produce specular reflections at boundaries, bringing segmentation difficulties. As presented in \cref{table:TODQuanti,fig:qualitative}, \methodname{} outperforms all competing methods, confirming Retinex effectiveness for transparent objects.

\noindent\textbf{Concealed defect detection.} 
Concealed defects on textured industrial surfaces (CDS2K~\cite{fan2023advances}) involve directional lighting interacting with material variations. As shown in \cref{table:CDDQuanti,fig:qualitative}, \methodname{} outperforms existing methods, indicating that the Retinex principle transfers to industrial settings.

\textbf{Qualitative analysis.} Fig.~\ref{fig:qualitative} shows comparisons across the four COS sub-tasks. The decompositions make per-task physical mechanisms visible: \eg, for polyps, $L$ captures the endoscope's point-source falloff while $R$ cleanly exposes vascularization that separates lesions from mucosa.

\textbf{Generalization Beyond COS.} We apply the Retinex strategy as a plug-in (``+R'') to existing methods on six segmentation tasks (Table~\ref{tab:generalization}). Strong gains on shadow detection (+5.00\% BER reduction) and infrared target detection (+2.50\% IoU) confirm that illumination-entangled tasks benefit most from Retinex augmentation. Modest but consistent gains on other tasks indicate broader generalization.

\vspace{-2mm}
\subsection{Ablation Studies}\vspace{-1mm}
\label{sec:ablation}

All ablations and in-depth analysis use the {ResNet50 backbone on COD10K} unless stated.

\textbf{Component contributions.} As shown in rows 1-5 of Table~\ref{tab:ablation}, each component yields a consistent gain over the TRD-only configuration, including discriminability-guided fusion, reflectance-space contrast, clean edge attribution, and reflectance-boundary alignment. 
\begin{figure}[t]
\setlength{\abovecaptionskip}{0cm}
\centering
\includegraphics[width=\linewidth]{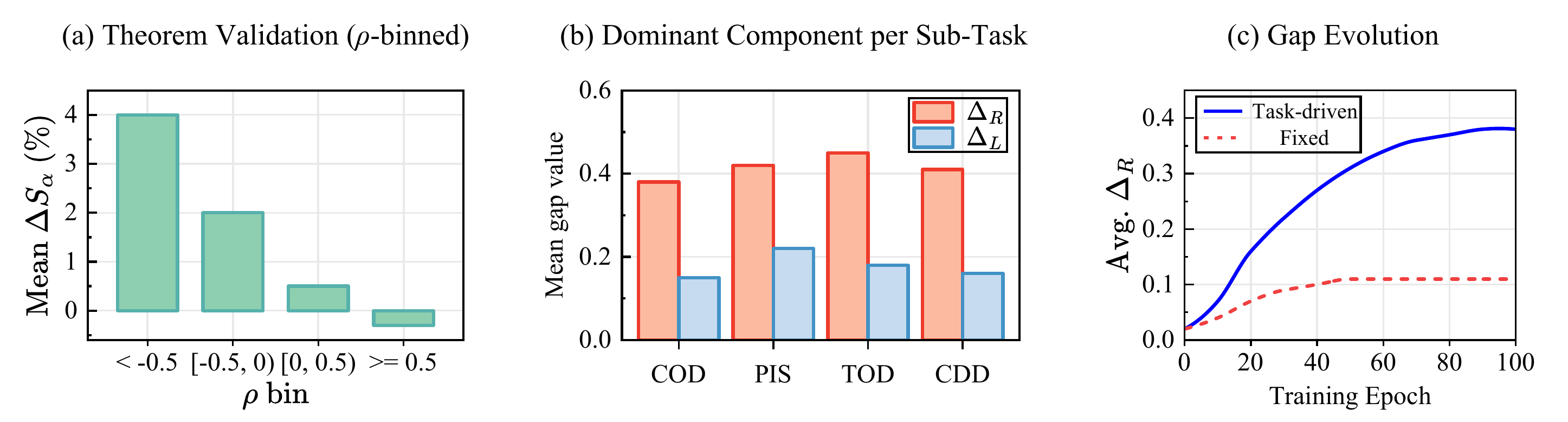}
\caption{{In-depth analysis.} (a) $\Delta S_\alpha$ vs.\ $\rho$ bins, validating Theorem~\ref{thm:main}. (b) Dominant component across COS sub-tasks. (c) Task-driven decomposition yields larger gaps than fixed Retinex.}
\label{fig:analysis}
\vspace{-5mm}
\end{figure}
 
\textbf{Decomposition family.} As presented in rows 6-9 of Table~\ref{tab:ablation}, holding the rest of the architecture fixed, swapping TRD for FFT, Wavelet, DCT, or a fixed RetinexNet (ME and RB disabled) loses $0.011$--$0.017$ $S_\alpha$. Even the fixed Retinex variant trails task-driven TRD by $0.011$, indicating that both the homogeneous Retinex factorization and end-to-end task-driven optimization are jointly essential.
 
\textbf{DGA design choices.} In Table~\ref{tab:ablationDGA}, DGA ($k=7$ with explicit gap) outperforms standard fusion baselines (concatenation, SE, cross attention) by $0.6$--$1.1\%$ $S_\alpha$. Removing the gap term and using raw variance alone loses $0.7\%$, confirming that the relative comparison against the composite carries the discriminative signal. DGA is robust to $k \in \{3,5,7,11\}$ within $0.4\%$ $S_\alpha$.
 
\textbf{Contrastive loss space.} $I$-space learning \emph{degrades} performance (Table~\ref{tab:cl_space}) because visual entanglement enforces $f_I^+ \approx f_I^-$ and drives feature collapse. $L$-space helps mildly, while $R$-space yields the largest gain. $I + R$ space underperforms $R$ alone, as the $I$-space reintroduces the collapse signal.
 
\textbf{Alternative homogeneous decompositions.} As indicated in Table~\ref{tab:alt_decomp}, edge-preserving filters (Guided Filter, Bilateral) and Gray-World color constancy yield only marginal gains over the baseline. Retinex outperforms all alternatives by $1.8$--$2.2\%$, as none of these alternatives admit a principled discriminability mechanism nor end-to-end joint optimization with the segmentation objective.

\vspace{-2mm}
\subsection{In-Depth Analysis}
\label{sec:analysis}\vspace{-1mm}

\textbf{Validating the Discriminability Gap Theorem.} As shown in Fig.~\ref{fig:analysis}(a), for each of the 2026 \textit{COD10K} test images we estimate $\rho = \cos\theta(\delta_L, \delta_R)$ from the learned $L, R$ maps and bin images into four $\rho$ ranges. Mean $\Delta S_\alpha$ over the no-decomposition baseline ($S_\alpha = 0.812$) decreases monotonically with $\rho$: $+4.0\%$ for $\rho < -0.5$ (20\% of images), $+2.0\%$ for $-0.5 \leq \rho < 0$ (38\%), $+0.5\%$ for $0 \leq \rho < 0.5$ (30\%), and $-0.3\%$ for $\rho \geq 0.5$ (12\%). The continuous Pearson correlation across all 2026 images is $r = -0.67$ (95\% CI $[-0.70, -0.64]$, $p < 10^{-3}$), confirming that decomposition gain scales with anti-correlation as Theorem~\ref{thm:main} predicts. The mild negative mean in the $\rho \geq 0.5$ bin lies within per-image variance and reflects boundary regions where (A2) is weakly violated; the bound itself saturates to unity in this aligned regime.
 
\textbf{Per-sub-task dominant component.} As presented in Fig.~\ref{fig:analysis} (b), for each COS sub-task we sample 200 test images and report the mean $\Delta_R$ and $\Delta_L$ at the second-deepest decoder level. The dominant component varies systematically with the underlying physical mechanism in \S\ref{sec:theory}: COD ($\Delta_R = 0.38$ vs.\ $\Delta_L = 0.15$) and PIS ($0.42$ vs.\ $0.22$) are reflectance-dominated, since material/tissue differences carry the signal under equalized composite appearance; 
TOD and CDD show the similar trends.

\textbf{Task-driven decomposition amplifies the gap (Fig.~\ref{fig:analysis}(c)).} We track the mean $\Delta_R$ on a fixed validation subset every 10 epochs. Task-driven TRD grows the gap from $0.02$ to $\sim$0.38 (saturating around epoch 60--70), whereas a frozen RetinexNet~\cite{wei2018deep} plateaus at $\sim$0.11 within 30 epochs. The $3.5\times$ ratio at convergence is direct evidence that segmentation gradients actively reshape the decomposition toward discriminative factorizations, which neither fixed Retinex nor pure reconstruction can deliver.

\vspace{-2mm}
\section{Discussion}\label{sec:discussion} \vspace{-1mm}
% See limitations and future directions in Appendix \ref{sec:limitation}.
 
\textbf{Retinex as an inductive bias in the foundation-model era.} Even on top of strong foundation backbones (DINOv2-L), Retinex decomposition delivers additional gains over composite-space representations alone. Foundation models scale representations within the composite space, whereas Retinex restructures the observation space before feature extraction. The two are complementary rather than competing, and the auxiliary modules of \methodname{} maintain a constant absolute cost regardless of backbone scale, making the framework increasingly attractive as backbones grow.

\vspace{-2mm}
\section{Conclusion}\vspace{-1mm}
\label{sec:conclusion} 
We present a homogeneous--heterogeneous decomposition framework for COS and show that Retinex-based decomposition helps break camouflage by exposing discriminative differences. Based on this insight, \methodname{} achieves SOTA performance on four benchmarks. The Discriminability Gap Theorem provides insight into when the method works, offering a principled perspective on decomposition-based dense prediction. Extensive experiments validate the effectiveness and potential of our RIDE.

\newpage
\bibliographystyle{unsrt}
\bibliography{egbib}

\end{document}